\documentclass[runningheads]{llncs}

\usepackage{eccv}

\usepackage{eccvabbrv}

\usepackage{graphicx}
\usepackage{booktabs}

\usepackage[accsupp]{axessibility}  

\usepackage{hyperref}

\usepackage{orcidlink}

\usepackage{latexsym}
\usepackage{amssymb}
\usepackage{amsmath}
\usepackage{enumitem}
\usepackage{color}
\usepackage{algorithm}
\usepackage{algpseudocode}
\usepackage{algorithmicx}
\usepackage{xcolor}
\usepackage{multirow}
\usepackage{marginnote}
\usepackage{mathtools}
\usepackage{tikz}
\usetikzlibrary{shapes.geometric, arrows, positioning}
\usepackage{pgfplots}
\pgfplotsset{compat=1.18}
\usepgfplotslibrary{fillbetween} 
\usepackage{bm}  
\usepackage{adjustbox}
\usepackage{makecell}

\newcommand{\predictedbox}{b}

\newcommand{\set}[1]{\left\{#1\right\}}
\DeclareMathOperator*{\argmax}{arg\,max}

\newcommand{\ioucert}{{\sc \sf IoUCert}}

\begin{document}

\title{IoUCert: Robustness Verification for Anchor-based Object
Detectors}


\author{Benedikt Brückner\inst{*,1}\orcidlink{0000-0003-0699-1688} \and
Alejandro J.~Mercado\inst{*,\dagger,2}\orcidlink{0009-0009-3905-4579} \and
Yanghao Zhang\inst{1}\orcidlink{0000-0002-8499-0974} \and
Panagiotis Kouvaros\inst{1}\orcidlink{0000-0003-2697-0710} \and
Alessio Lomuscio\inst{1,2}\orcidlink{0000-0003-3420-723X}
}

\authorrunning{B.~Brückner et al.}

\institute{Safe Intelligence\\
\email{\{benedikt,yanghao,panagiotis,alessio\}@safeintelligence.ai}
\and
Imperial College London\\
\email{a.mercado24@imperial.ac.uk}
}

\maketitle
\begingroup
\renewcommand\thefootnote{}
\footnotetext{$^*$Equal contribution.}
\footnotetext{$^\dagger$Work done during an internship at Safe Intelligence.}
\endgroup

\begin{abstract}
While formal robustness verification has seen significant success in image classification, scaling these guarantees to object detection remains notoriously difficult due to complex non-linear coordinate transformations and Intersection-over-Union (IoU) metrics. As a fundamental step towards verifying complete detection pipelines, we introduce \ioucert, a novel formal verification framework designed specifically to overcome these core mathematical bottlenecks. By isolating the object localisation task in single-object settings, we propose a coordinate transformation that circumvents precision-degrading relaxations of non-linear box prediction functions. This approach allows us to optimise bounds directly with respect to anchor box offsets, enabling a novel Interval Bound Propagation method that derives optimal IoU bounds. We demonstrate that \ioucert{} enables, for the first time, the robustness verification of foundational, anchor-based architectures including tractable variants of SSD, YOLOv2, and YOLOv3 against various input perturbations, providing a rigorous theoretical basis for future end-to-end detector verification.
\end{abstract}

\section{Introduction}
\label{sec:introduction}

Neural networks are increasingly being deployed in safety-critical domains
such as autonomous vehicles~\cite{Cao+19} and medical
diagnostics~\cite{Litjens+17}. Even when these models produce correct
predictions for a given input, they can make incorrect predictions on
perceptually equivalent variations of that
input~\cite{GoodfellowShlensSzegedy15,AmirkhaniKarimiBanitalebiDehkordi23,Zhang+24,ZhangWangRuan21a,Wang+25a}.
Formal verification approaches address this by assessing the robustness of
models with respect to various input perturbations such as white
noise~\cite{Katz+19,Singh+19a,Wang+21b,KouvarosLomuscio21},
photometric~\cite{Kouvaros+21,Henriksen+21},
geometric~\cite{Balunovic+19a,Batten+24} and
convolutional~\cite{MziouSallamiAdjed22,BruecknerLomuscio25} perturbations.

While significant progress in the area has been made, most methods target
image
classifiers~\cite{FerlezKhedrShoukry22,Wu+24,Duong+23,Bak21,Lopez+23,LemesleLehmannGall24,Demarchi+24,Althoff15,ZhangKouvarosLomuscio25}.
As a result, the methods cannot analyse complex architectures widely used
in computer vision applications that require robustness validation. A
notable example of this is the class of object detection (OD)
models~\cite{Redmon+16a,Liu+16,Zou+23} which may exhibit vulnerabilities
such as that illustrated in Figure~\ref{fig:od_examples}. The technical
challenge is that OD models include non-linear components such as
non-maximum suppression, additional logic components such as
Intersection-over-Union (IoU) calculations, and non-linear transformations
to convert the raw model outputs into bounding box predictions. Existing
robustness verifiers either do not support these or provide loose
approximations for
them~\cite{Cohen+24,Raviv+24,ChowdhuryKhandelwalDSouza25,NiralaSarkar25}.
Most work on OD verification employs regression models predicting
four box coordinates which lack the typical backbone, neck and head
structure of object detectors, do not consider heads operating at multiple
grid scales, and use shallow
backbones~\cite{Cohen+24,Raviv+24,ChowdhuryKhandelwalDSouza25,NiralaSarkar25}.
They therefore lack the capabilities to analyse modern OD models.

\begin{figure}[tbp]
\centering
\includegraphics[width=0.2\textwidth]{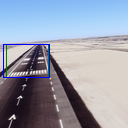}
\hspace{0.03\textwidth} 
\includegraphics[width=0.2\textwidth]{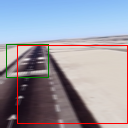}
\caption{A counterexample from robustness verification of a YOLOv3 model on
a runway detection task. Left: original image; right: perturbed image.
Green boxes: ground truth, blue box: original model prediction (IoU
$=0.89$), red box: prediction on the perturbed image (IoU $=0.11$).}
\label{fig:od_examples}
\end{figure}

In this paper, we overcome these limitations by introducing \ioucert{},
a method for the verification of OD models that is provably tighter than
previous approaches, supports foundational anchor-based architectures such as SSD
and YOLO, and scales beyond simplified toy models.
Specifically, we make the following contributions:
\begin{itemize}
\item  We improve existing Interval Bound Propagation (IBP) methods by 
deriving optimal IoU bounds for anchor-based object detection models.
\item We introduce an architecture-aware verification framework for anchor-based object detectors. By using a novel coordinate transformation, we bypass the relaxation of complex non-linearities (such as class logit sigmoids), effectively avoiding the computational bottlenecks of standard verifiers. This allows us to scale efficiently and enables the first formal verification of models like SSD and YOLOv3 in single-object settings.
\item We derive optimal linear relaxations for LeakyReLU activations in
YOLOv3 models to minimise relaxation errors. 
\item We integrate \ioucert{} into the {\sf
Venus} verifier to analyse the robustness of SSD,
YOLOv2 and YOLOv3 models on datasets of varying complexity, including Pascal VOC~\cite{Everingham+10}, COCO~\cite{Lin+14} and a runway detection task on LARD~\cite{Ducoffe+23}.
\end{itemize}

Bridging formal verification and object detection requires managing immense computational complexity: complete verification provides guarantees over infinite continuous perturbation spaces, far more demanding than empirical testing. To remain tractable for realistic architectures, we focus on the single-object setting, verifying the core classification-regression mechanisms while leaving the combinatorial complexity of multi-object competition and Non-Maximum Suppression (NMS) for future work.

The rest of the paper is organised as follows: After discussing
related work below, we discuss background material on OD models and
neural network verification in Section~\ref{sec:background} and
present \ioucert{} in Section~\ref{sec:method}, which we evaluate in
Section~\ref{sec:evaluation}. We conclude in
Section~\ref{sec:conclusion}.

\section{Related Work}
\label{sec:related_work}

Neural network verification encompasses complete methods relying on Satisfiability Modulo
Theories~\cite{PulinaTacchella12,Ehlers17,Katz+17,Katz+19} or Mixed-Integer Linear Programming solvers~\cite{LomuscioMaganti17,TjengXiaoTedrake19,Anderson+20,Singh+19b,Botoeva+20,KouvarosLomuscio21}, and incomplete methods that use relaxations like Semidefinite Programming~\cite{RaghunathanSteinhardtLiang18,FazlyabMorariPappas20,LanBruecknerLomuscio23,ChiuZhang23} or efficient bound propagation~\cite{Gowal+19,Wang+18b,Singh+18,Zhang+18,Wang+21b,Zhang+22,Kouvaros+25}. State-of-the-art approaches often combine GPU-accelerated incomplete bound propagation with Branch-and-Bound heuristics to achieve completeness and scalability~\cite{Wang+18a,DePalma+21,HenriksenLomuscio21,Brix+24,Zhou+24,Zhou+25}.

While verification has seen significant progress, most verifiers primarily target image classifiers~\cite{Brix+23b}. Recent efforts in object detection (OD) verification struggle to scale beyond highly simplified setups. For instance, \cite{Cohen+24} bounds the Intersection-over-Union (IoU) metric via Interval Bound Propagation (IBP) on predicted corner coordinates, while \cite{Raviv+24} encodes IoU as a network layer but struggles with loose bounds on operations like \textit{max}, \textit{min}, and division. Both methods demonstrate formal verification on simple regression toy models outputting four corner coordinates for a single object, and these models achieve low accuracies. The probabilistic verifier by \cite{Liu+26} scales to larger architectures, but suffers from unsound robustness certificates, even for small models. Other works based on ImageStars~\cite{ChowdhuryKhandelwalDSouza25} or branch-free IBP~\cite{NiralaSarkar25} are similarly evaluated on toy models and lack support for the multi-scale heads, non-linear coordinate transformations, and anchor-based structures of real-world detectors like SSD~\cite{Liu+16} and YOLO~\cite{Redmon+16a}.

In contrast, \ioucert{} is a sound verifier that provides an object detection-aware framework which explicitly exploits the specific structure of anchor-based detectors to efficiently scale verification. Unlike previous approaches that naively apply image verifiers to OD models which often incur unnecessary relaxations that lead to timeouts, we introduce a coordinate transformation to directly operate on inferred offset bounds. This enables optimal bounding for realistic anchor-based detectors that achieve high accuracies. Combined with tight symbolic bounding and optimal IoU relaxations, our method scales to architectures of practical relevance. For the first time, we demonstrate robustness verification on SSD~\cite{Liu+16}, YOLOv2~\cite{RedmonFarhadi17}, and YOLOv3~\cite{RedmonFarhadi18} models. While newer YOLO variants exist, they employ components that remain highly challenging for current verifiers, such as attention layers~\cite{BochkovskiyWangLiao20}.

\section{Background}
\label{sec:background}
We formally define the single-object detection task, describe typical anchor-based architectures, and outline the neural network verification problem. Extending this framework to multiple objects is left for future work due to the added complexity.

\subsection{Object detection}
\label{ssec:object_detection}
Let $I \in \mathbb{R}^{C \times H \times W}$ be an image that is annotated with a ground truth box $g = (z_0, z_1, z_2, z_3, g_c)$, where $(z_0, z_1)$ and $(z_2, z_3)$ are its top-left and bottom-right corners, and $g_c$ is its class. For a single-object task the expected output is exactly one bounding box matching this ground truth, but regressing a single set of coordinates performs poorly. Instead, a modern object detector $\mathbf{O}$ predicts a large set of candidate detections $\mathbf{O}(I) = \{ \predictedbox^1, \dots, \predictedbox^k \}$ and relies on post-processing to isolate the final prediction, which makes the naive application of robustness verifiers prohibitively difficult.

Each predicted box $\predictedbox = (z_0, z_1, z_2, z_3, \mathbf{c}, s)$ contains corner coordinates, class probabilities $\mathbf{c} \in \mathbb{R}^{n_c}$, and a confidence score $s \in [0, 1]$. Let the area of box $b$ be $\mathbf{a}(b) = (b_3 - b_1)(b_2 - b_0)$ and the intersection of two boxes $b, b'$ be $\mathbf{i}(b, b') = (\max(b_0, b'_0), \max(b_1, b'_1), \min(b_2, b'_2), \min(b_3, b'_3))$. The Intersection-over-Union is $\mathrm{IoU}(b, b') = \mathbf{a}(\mathbf{i}(b, b')) / (\mathbf{a}(b) + \mathbf{a}(b') - \mathbf{a}(\mathbf{i}(b, b')))$.

An OD model is \emph{correct} with respect to ground truth $g$ and thresholds $\tau_{\mathrm{iou}}, \tau_{\mathrm{class}}$ if: (1) it outputs exactly one box $b$ after post-processing, (2) the predicted class $\argmax_{i} \mathbf{c}_i$ matches $g_c$ and its class score exceeds $\tau_{\mathrm{class}}$, and (3) $\mathrm{IoU}(b, g) \geq \tau_{\mathrm{iou}}$.

Modern anchor-based OD models (\eg, SSD, YOLO) predict offsets for a fixed set of prior \emph{anchor boxes} $\{p^1, \dots, p^n\}$ via a three-step pipeline:

\noindent\textbf{Step 1:~Offset Prediction.} The model predicts $\mathcal{A} = \{a^1, \dots, a^n\}$, where $a^i = (o^i, l^i)$ contains coordinate offsets $o^i$ and logits $l^i$ for class and confidence scores.

\noindent\textbf{Step 2:~Box Construction.} Offsets $o^i$ and anchors $p^i$ are combined via a model-specific function $\phi(o^i, p^i) = (c_x, c_y, w, h)$ into centre-format coordinates (see Appendix~\ref{sec:h_phi_functions}). This is mapped to the corner-format $(z_0, z_1, z_2, z_3)$ using $\psi(\phi(o^i, p^i)) = (c_x - \frac{w}{2}, c_y - \frac{h}{2}, c_x + \frac{w}{2}, c_y + \frac{h}{2})$. The inverse, used later in our verification procedure, is $\psi^{-1}(z_0, z_1, z_2, z_3) = (\frac{z_0 + z_2}{2}, \frac{z_1 + z_3}{2}, z_2 - z_0, z_3 - z_1)$. Each box is assigned a class and confidence score from the corresponding logits $l^i$.

\noindent\textbf{Step 3:~Post-processing.} Non-maximum suppression (NMS) filters overlapping or low-confidence boxes. For single object detection, we assume this selects the highest-confidence bounding box above a threshold (this is without loss of generality, see Appendix~\ref{sec:proof_single_highest_box}). We keep this threshold even though we deal with single-object detection, since it allows a detector to abstain when all candidates score below it. Verifying correctness therefore requires certifying both the selection of the top box and the satisfaction of the threshold.

\subsection{Neural Network Verification}
\label{ssec:neural_network_verification}
Verification aims to certify a network's robustness. A network $N$ is \emph{robust} to input constraints $\zeta_x$ (\eg, $\ell_\infty$ bounds, brightness, or blurring~\cite{Kouvaros+21,BruecknerLomuscio25}) if $N(x)$ satisfies output constraints $\zeta_y$ for all $x \in \zeta_x$. Here, $\zeta_y$ encodes OD correctness as defined above.

Interval Bound Propagation (IBP)~\cite{Gowal+19} is a fast, incomplete method that propagates concrete intervals, but it often yields loose bounds as it cannot capture variable dependencies~\cite{Wang+18b}. Symbolic Interval Propagation (SIP) mitigates this by propagating symbolic bounding equations. The CROWN\slash DeepPoly\slash back-substitution method~\cite{Zhang+18,Singh+19a,HenriksenLomuscio21} propagates symbolically from the layer of interest back to the input to fully capture dependencies and obtain tighter bounds. When bounds remain too loose, branch-and-bound partitions the problem space to reduce relaxation errors. This is done via \textit{input splitting} for low input dimensions~\cite{Wang+18b,Botoeva+20} or \textit{neuron splitting} for high-dimensional problems~\cite{Botoeva+20,Ferrari+22}.

Verifying OD models with standard SIP is challenging because verifiers typically target ReLU networks (YOLOv3 uses LeakyReLU) and require linear relaxations for all non-linear components (\eg, box construction). In Section~\ref{sec:method}, we overcome these limitations by proposing a coordinate transformation to bypass certain non-linearities, derive optimal IoU bounds, and define optimal LeakyReLU relaxations.

\section{Method}
\label{sec:method}
Verifying object detectors is challenging: modern OD models combine large
architectures with non-linear functions mapping offsets to box coordinates,
and existing verifiers fail to scale to SSD or YOLO due to loose bounds and
missing support for anchor boxes, multiple prediction heads, and confidence
scores. To the best of our knowledge, \ioucert{} is the first to support
such coordinate transformations, enabling IoU bounding for detectors with
non-linear box predictions, for which we further derive the provably tightest
IoU bounds. Figure~\ref{fig:od_pipeline} presents an overview of our verification
pipeline.
\tikzstyle{block} = [rectangle, rounded corners, minimum width=2.5cm, minimum height=0.8cm, text centered, align=center, draw=black, fill=blue!15]
\tikzstyle{nnblock} = [rectangle, rounded corners, minimum width=2.5cm, minimum height=0.8cm, text centered, align=center, draw=black, fill=green!20]
\tikzstyle{odblock} = [rectangle, rounded corners, minimum width=2.5cm, minimum height=0.8cm, text centered, align=center, draw=black, fill=yellow!20]
\tikzstyle{finalblock} = [rectangle, rounded corners, minimum width=2.5cm, minimum height=0.8cm, text centered, align=center, draw=black, fill=red!15]
\tikzstyle{module} = [rectangle, dashed, draw=black]
\tikzstyle{arrow} = [thick,->,>=stealth]
\begin{figure}[t]
\centering

\resizebox{\columnwidth}{!}{
\begin{tikzpicture}[node distance=.5cm]

\node (constraints) [block] {Input\\Constraints \\ $( \zeta_x)$};
\node (splitting) [block, below of=constraints, yshift=-0.8cm] {Image};

\node (nn) [nnblock, right of=constraints, xshift=3cm, yshift=-0.6cm] {Neural\\Network\\(FFNN)};
\node (offsetbox) [module, minimum width=3cm, minimum height=2cm, right of=constraints, xshift=3cm, yshift=-0.6cm] {};

\node (odbox) [module, minimum width=3cm, minimum height=3cm, right of=nn, xshift=3cm] {};

\node (candidates) [odblock, right of=nn, xshift=3cm, yshift=0.7cm] {Candidate Box\\Selection};
\node (iou) [odblock, below of=candidates, yshift=-0.8cm] {IoU\\Bounds};

\node (decision) [finalblock, right of=odbox, xshift=3cm] {\textit{ROBUST},\\ \textit{NONROBUST}\\or \textit{UNKNOWN}};

\draw [arrow] (constraints.east) -- ++(.4, 0) |- (offsetbox.west);
\draw [arrow] (splitting.east) -- ++(.4, 0) |- (offsetbox.west);
\draw [arrow] (offsetbox.east) -- (odbox.west);
\draw [arrow] (candidates) -- (iou);
\draw [arrow] (odbox.east) -- (decision.west);

\node at (offsetbox.north) [above, yshift=0.1cm, font=\bfseries] {NN Verifier Bounds};
\node at (odbox.north) [above, yshift=0.1cm, font=\bfseries] {OD Verification Module};

\end{tikzpicture}
}
\caption{Pipeline of our object detection (OD) verification framework,
combining verifier bounds, candidate selection, and optimal IoU bound
derivation to reach \textit{ROBUST}, \textit{NONROBUST} or \textit{UNKNOWN} conclusions.}
\label{fig:od_pipeline}
\end{figure}

\subsection{Coordinate Transformation}
\label{ssec:coordinate_transformation}
As outlined in Section~\ref{ssec:object_detection}, object detectors
predict offsets $\mathbf{o}$ that are converted to corner coordinates
$(z_0, z_1, z_2, z_3)$ via the mapping $\psi \circ \phi$. We bound the IoU
function \emph{directly} with respect to the predicted offsets, rather than
propagating the offset bounds through $\psi \circ \phi$ before bounding the IoU
function~\cite{Cohen+24}, thereby circumventing the overapproximation that
bound propagation over $\psi \circ \phi$ would induce.

To bound the IoU between the predicted offsets $\mathbf{o}$ of a bounding
box and a fixed ground truth box $g$, we need to bound the expression
$\mathrm{IoU}(\psi \circ \phi(\mathbf{o}), g)$. Here, we focus on deriving the
optimal upper bound using the given offset bounds $[\underline{\mathbf{o}},
\overline{\mathbf{o}}]$ for $ \mathbf{o} = (o_0, o_1, o_2, o_3)$; the
derivation for the lower bound is analogous.
We want to solve the constrained maximisation problem:
\begin{equation}
\label{eq:original_constraints}
\begin{gathered}
\begin{aligned}[t]
& \max_{\mathbf{o} = (o_0, o_1, o_2, o_3)} \mathrm{IoU}(\psi \circ \phi(\mathbf{o}), g) \\
& \text{s.t.} \quad \underline{\mathbf{o}} \leq \mathbf{o} \leq \overline{\mathbf{o}},
\end{aligned}
\end{gathered}
\end{equation}
which is difficult to do using linear-relaxation-based bound propagation
methods due to the non-linearity of $\psi \circ \phi$. We define
$\gamma(\mathbf{o})=\psi \circ \phi(\mathbf{o})$ which maps from offset space
to corner space. We show the injectivity of the mappings $\psi, \phi$ in
Appendix~\ref{sec:h_phi_functions}. Since the composition of injective
functions is injective~\cite[Theorem 12.2]{Hammack18}, it follows that the
composition $\psi \circ \phi$ is also injective. This allows us to define
$\left (\psi \circ \phi \right )^{-1}(\mathbf{z}) =
\phi^{-1} \circ \psi^{-1}(\mathbf{z})$ as the mapping from corner space to
offset space. We substitute $\mathbf{o} = \phi^{-1} \circ \psi^{-1}(\mathbf{z})$ in
Problem~\ref{eq:original_constraints}.
The objective function becomes
$\mathrm{IoU}(\psi \circ \phi(\phi^{-1} \circ \psi^{-1} (\mathbf{z})), g) =
\mathrm{IoU}(\mathbf{z}, g)$ while the constraints become
$\underline{\mathbf{o}} \leq \phi^{-1} \circ \psi^{-1}(\mathbf{z}) \leq
\overline{\mathbf{o}}$. In summary, we obtain the following equivalent
optimisation problem which directly optimises over the corner coordinates
$\mathbf{z} = (z_0, z_1, z_2,
z_3)$~\cite[pp.~130--131]{BoydVandenberghe04}:
\begin{equation*}
\begin{aligned}
& \max_{\mathbf{z} = (z_0, z_1, z_2, z_3)} \mathrm{IoU}(\mathbf{z}, g) \\
& \text{s.t.} \quad \underline{\mathbf{o}} \leq \phi^{-1} \circ \psi^{-1}(\mathbf{z}) \leq \overline{\mathbf{o}},
\end{aligned}
\end{equation*}
The inverse mappings ensure that the feasible region, originally
defined in the offset space, is correctly expressed over $\mathbf{z}$.
Rewriting the constraints explicitly in terms of the corner coordinates, we
obtain:
\begin{equation}
\label{eq:new_constraints}
\begin{gathered}
\max_{z_0,z_1,z_2,z_3} \text{IoU}((z_0,z_1,z_2,z_3),g) \\
\begin{aligned}
\text{s.t.}\quad 
2 c_x( \underline{\mathbf{o}}_0) &\leq z_0 + z_2 \leq 2 c_x( \overline{\mathbf{o}}_0), \\
2 c_y( \underline{\mathbf{o}}_1) &\leq z_1 + z_3 \leq 2 c_y( \overline{\mathbf{o}}_1), \\
w( \underline{\mathbf{o}}_2) &\leq z_2 - z_0 \leq w( \overline{\mathbf{o}}_2),  \\
h( \underline{\mathbf{o}}_3) &\leq z_3 - z_1 \leq h( \overline{\mathbf{o}}_3), 
\end{aligned}
\end{gathered}
\end{equation}
where $c_x$ and $c_y$ denote the centre coordinate transformations,
and $w$ and $h$ represent the width and height, each expressed in
terms of their respective offsets. More generally, the transformation
applies whenever the box-construction map $\psi \circ \phi$ is injective with
a tractable inverse, \ie when the decoding function is strictly monotonic in
each offset (Appendix~\ref{sec:h_phi_functions}). The applicability
of the method to other detector families is discussed in
Appendix~\ref{sec:scope_and_extensions}.

\subsection{Optimal IoU IBP Bounds}
\label{ssec:optimal_iou_ibp_bounds}
Extreme points for the $\text{IoU}$ in Problem~\ref{eq:new_constraints} exist
either (i) where the gradient of the function is zero; (ii) along the
borders of the constrained region; or (iii) where the function is not
differentiable. Cohen et al.~\cite{Cohen+24} show that the partial derivatives of
$\text{IoU}((z_0, z_1, z_2, z_3),g)$ are never zero within the feasible
region, so we only consider cases (ii) and (iii). Since each constraint
only includes either the variables $(z_0, z_2)$ or $(z_1, z_3)$, we
consider the 2D plane $(z_0, z_2)$ corresponding to width-related
variables, and the plane $(z_1, z_3)$ corresponding to height-related
variables separately. We identify the coordinates within each 2D plane, and
then combine the coordinates from both planes to form the complete
candidate extreme points. We use $L_i$ and $U_i$ to denote the lower and
upper bounds for the $i$-th constraint in
Problem~\ref{eq:new_constraints}.
\begin{enumerate}
\item \textbf{\boldmath{Corner points.}} The corners of each region where
the constraints intersect form our first set of critical points over the
boundaries. These are given by
\begin{align*}
P^{c}_x = & \left\{ \left(\frac{U_0 - U_2}{2},\frac{U_0 +
        U_2}{2}\right),
\left(\frac{U_0 - L_2}{2}, \frac{U_0 + L_2}{2}\right),
\right. \\
 & \left.\left(\frac{L_0 - U_2}{2}, \frac{L_0 + U_2}{2}\right),
\left(\frac{L_0 - L_2}{2}, \frac{L_0 + L_2}{2}\right) \right\}
\end{align*}
over the $(z_0, z_2)$ plane, and
\begin{align*}
    P^{c}_y = & \left \{\left(\frac{U_1 - U_3}{2}, \frac{U_1
        + U_3}{2}\right), 
\left(\frac{U_1 - L_3}{2}, \frac{U_1 + L_3}{2}\right),
\right. \\
& \left. \left(\frac{L_1 - U_3}{2}, \frac{L_1 +
    U_3}{2}\right), 
\left(\frac{L_1 - L_3}{2}, \frac{L_1 + L_3}{2}\right)\right \}
\end{align*}
over the $(z_1, z_3)$ plane.
\item  \textbf{\boldmath{Stationary points on boundaries.}}
The second set of critical points are those where the gradient along the
boundaries is zero. Each boundary of the feasible region is defined by one
of the following equations (with $k = 0$ or $k = 1$, depending on the
plane):
\begin{enumerate}
\item $z_k + z_{k+2} = U_k$ or $z_k + z_{k+2} = L_k$
\item $z_{k+2} - z_k = U_{k+2}$ or $z_{k+2} - z_k = L_{k+2}$.
\end{enumerate}
As we show in Appendix~\ref{ssec:iou_on_the_border_of_constraints}, the
gradient along these boundaries, where it exists, is either zero at every
point or non-zero at every point. When the gradient is always zero, the
value on the boundary matches that at the corners of the region. If it is
always non-zero, the extremum is attained at one of the ends of the corner
segment. Consequently, these stationary points on the boundaries are
already included in $P^c_x$, $P^c_y$.
\item \textbf{Non-differentiable points.} Since the
IoU function is non-differentiable whenever $z_i = g_i
$, the third set of critical points includes
\begin{itemize} 
\item the points at the
intersection of the ground truth box coordinates
with each border (when they intersect):
$P^{\text{int}}_x= \{ 
(g_0, U_0 - g_0),\allowbreak
(g_0, L_0 -g_0),\allowbreak
(g_0, U_2 + g_0),\allowbreak
(g_0, L_2 + g_0),\allowbreak
(U_0-g_2,g_2),\allowbreak
(L_0-g_2,g_2),\allowbreak
(g_2-U_2,g_2),\allowbreak
(g_2-L_2,g_2)\}$ along the $(z_0, z_2)$ plane, and 
$P^{\text{int}}_y=\{(g_1, U_1 - g_1),\allowbreak
(g_1, L_1 - g_1),\allowbreak
(g_1, U_3 + g_1),\allowbreak
(g_1, L_3 + g_1),\allowbreak
(U_1 - g_3, g_3),\allowbreak
(L_1 - g_3, g_3),\allowbreak
(g_3 - U_3, g_3),\allowbreak
(g_3 - L_3, g_3)\}$ along the $(z_1, z_3)$ plane.
\item the corners of
the ground truth boxes when they fall within the
constraint region: $P^{gt}_x=(g_0$,$g_2)$ and
$P^{gt}_y=(g_1$,$g_3)$.
\end{itemize}
\end{enumerate}
For $i \in \set{0, 2}$, let $P_{z_i} = \{\alpha \mid
(\alpha,\beta) \in P^c_x \cup P^{\text{int}}_x \cup P^{gt}_x
\text{ for some } \beta\}$, and similarly, for $i \in
\set{1, 3}$, let   $P_{z_i} = \{\alpha \mid (\alpha,\beta)
\in P^c_y \cup P^{\text{int}}_y \cup P^{gt}_y \text{ for
some } \beta\}$. These sets of points can be used to optimise the IoU
function.
\begin{theorem} \label{th:iou_is_maximum}
The IoU function is maximum for a point obtained from
within the set $C_s=\bigtimes_i P_{z_i}$.
\end{theorem}
\begin{proof}
See Appendix~\ref{sec:proof_iou_is_maximum}.
\end{proof}
Following Theorem~\ref{th:iou_is_maximum}, the maximum and minimum IoU
values can be found by iterating through the critical points, checking
whether they satisfy the constraints, evaluating their {\em validity}
(whether they define valid boxes satisfying $z_0<z_2$ and $z_1<z_3$), and
updating the maximum and minimum IoU values. The same construction applies
to each 2D plane: the candidate set comprises 13 points, namely the 4
corners of the feasible region, the 8 intersections ($2 \times 4$) of the
two ground truth coordinate lines with the four lines bounding the region,
and the single ground truth corner; these are continuous box coordinates
rather than discrete pixel locations (see Figure~\ref{fig:iou_intuition} for
an illustration of the $(z_0,z_2)$ plane, including the candidates that are
pruned as infeasible).
Hence $|P^c_x \cup P^{\text{int}}_x \cup P^{gt}_x| = |P^c_y \cup
P^{\text{int}}_y \cup P^{gt}_y| = 13$, giving $|C_s| = 13^2 = 169$
candidate points overall. Since the maximum and minimum IoU are attained at
such points and only a finite number (169) exist, the algorithm is correct
and will terminate in constant time. For details see Appendix
\ref{sec:proof_verification_algorithm_correctness}.

\begin{figure}[t]
\centering
\begin{tikzpicture}[scale=0.62,>=stealth]
\filldraw[fill=blue!8,draw=blue!55,thick] (3,9)--(5,7)--(2,4)--(0,6)--cycle;
\draw[blue!45,thin,dashed] (3,9)--(2.5,9.5);
\draw[blue!45,thin,dashed] (5,7)--(5.5,6.5);
\draw[blue!45,thin,dashed] (0,6)--(-0.5,6.5);
\draw[blue!45,thin,dashed] (2,4)--(2.5,3.5);
\draw[red!75,dashed] (2.5,3.3)--(2.5,9.6);
\draw[red!75,dashed] (-1,6.5)--(5.7,6.5);
\node[red!75,font=\footnotesize] at (2.5,9.95) {$z_0=g_0$};
\node[red!75,font=\footnotesize,anchor=west] at (5.7,6.5) {$z_2=g_2$};
\foreach \p in {(3,9),(5,7),(2,4),(0,6)} \filldraw[blue!70] \p circle (4pt);
\foreach \p in {(2.5,8.5),(2.5,4.5),(0.5,6.5),(4.5,6.5)} \filldraw[orange] \p circle (4pt);
\foreach \p in {(2.5,9.5),(2.5,3.5),(-0.5,6.5),(5.5,6.5)} \draw[orange,thick] \p circle (4pt);
\filldraw[red] (2.5,6.5) circle (4pt);
\end{tikzpicture}
\caption{Geometric intuition for the IoU candidate set in the $(z_0,z_2)$
plane (the $(z_1,z_3)$ plane is analogous). The four constraints define a
feasible region (blue) with four corners (\textcolor{blue!70}{$\bullet$},
$P^c_x$), and the dashed ground truth coordinate lines $z_0=g_0$, $z_2=g_2$
mark the non-differentiability of the IoU, crossing at the ground truth
corner (\textcolor{red}{$\bullet$}, $P^{gt}_x$). Each ground truth line meets
each of the four lines bounding the region, giving $2 \times 4 = 8$
intersections ($P^{\text{int}}_x$): those on the feasible boundary
(\textcolor{orange}{$\bullet$}) are retained, while those outside it
(\textcolor{orange}{$\circ$}, where the ground truth lines meet the dashed
boundary-line extensions) are pruned by the feasibility check. This gives
$4 + 8 + 1 = 13$ enumerated candidates per plane (as used in
Theorem~\ref{th:iou_is_maximum}). Here $4$ of the $8$ intersections are
infeasible, leaving the $9$ filled points.}
\label{fig:iou_intuition}
\end{figure}

\subsection{Robustness Verification Algorithm}
Algorithm~\ref{alg:verification} introduces the \ioucert{} method for
establishing the robustness of an OD model for any input satisfying the specified
perturbation constraints.
\begin{algorithm}[t]
\caption{Robustness Verification for Object Detectors}
\label{alg:verification}
\begin{algorithmic}[1]
\State \textbf{Input:} Neural network $N$, input constraints $\zeta_x$, robustness constraints $\zeta_y$
\State \textbf{Output:} \textit{ROBUST} or \textit{NONROBUST} or \textit{UNKNOWN}

\Function{Verify}{$N, \zeta_x, \zeta_y$}
    \State bounds $\gets$ \Call{Bound\_Propagation}{$\zeta_x, N$}
    \State IoUBounds, scoreBounds, predClass $\gets$\\
    \hspace*{\algorithmicindent} \quad \Call{GetHighestBox}{bounds}
    \If{IoUBounds.min $\geq \tau_{\mathrm{iou}}$ \textbf{and} \\
        \hspace*{\algorithmicindent} \quad scoreBounds.lower $\geq \tau_{\mathrm{class}}$ \textbf{and}\\
        \hspace*{\algorithmicindent} \quad predClass $=$ \textit{class}}
        \State \Return \textit{ROBUST}
    \ElsIf{IoUBounds.max $< \tau_{\mathrm{iou}}$ \textbf{or}\\
       \hspace*{\algorithmicindent} \quad scoreBounds.upper $< \tau_{\mathrm{class}}$ \textbf{or}\\
       \hspace*{\algorithmicindent} \quad \Call{disagree}{predClass}}
        \State \Return \textit{NONROBUST}
    \Else
        \State \Return \textit{UNKNOWN}
    \EndIf
\EndFunction
\end{algorithmic}
\end{algorithm}
\ioucert{} employs existing bound propagation frameworks (such as IBP or
SIP) to obtain bounds on the output of the neural network component for a
given perturbation (Line~4). Given these bounds, it identifies all bounding
boxes that could potentially have the highest confidence score (Line~5).
Even though an OD model only outputs the bounding box with the highest
confidence score when there is only a single object, the approximate nature
of the bounds may cause the identification of multiple candidate boxes. For
example, if box~1 has score bounds $([0.5, 0.9])$ and box~2 has $([0.7,
0.8])$, either of them could be the top-scoring box.

To identify candidate boxes, \ioucert{} selects all boxes whose upper
bound on the confidence score exceeds the highest lower bound among all
boxes. This ensures that only boxes that could potentially be the
top-scoring one are considered. For each candidate, we compute
bounds on its class
scores and its IoU with the ground truth box. The detailed procedure is described in Appendix
\ref{sec:proof_verification_algorithm_correctness}.

If all candidate boxes meet the IoU threshold, the minimum confidence score
meets the class threshold, and all candidates agree on the predicted class,
the verification query is \textit{ROBUST}. If none of the candidates meets
the IoU threshold, or the maximum confidence score is below the
threshold, or all candidate boxes predict a class different from the
ground truth, the query is \textit{NONROBUST}. If the bounds are
too loose to determine the outcome, or if the candidate boxes do
not agree on a single class, \ioucert{} outputs \textit{UNKNOWN}.
\begin{theorem} \label{th:verification_algorithm_correctness}
\ioucert{} is correct. It is complete when
integrated with a branching framework.
\end{theorem}
\begin{proof}
See Appendix~\ref{sec:proof_verification_algorithm_correctness}.
\end{proof}
To determine the robustness of unknown cases, \ioucert{}
can be combined with any branching framework in neural network
verification, such as {\sc\sf Venus}~\cite{KouvarosLomuscio21}.

\subsection{Optimal Relaxations for LeakyReLU Activations}
\label{ssec:optimal_relaxations_leaky_relu}
While most neural network verification approaches focus on ReLU activation
functions, the YOLOv3 architecture employs LeakyReLU activations. We define
$\text{LeakyReLU}(x) = \max \{\alpha x, x\}$ with $\alpha \in [0, 1]$ and
concrete input bounds $x \in [l, u]$. If $u<0$ or $l>0$ the activation
function is said to be stable and can be represented exactly in a linear
bound propagation framework. For $l < 0 < u$, its behaviour is piece-wise
linear, and linear lower and upper bounding functions
$f_\text{lower}(x), f_\text{upper}(x)$ for it are given by
\begin{align}
f_\text{lower}(x) &= \tilde{\alpha} x, \quad \tilde{\alpha} \in [\alpha, 1] \\
f_\text{upper}(x) &= \frac{u-\alpha l}{u-l}x + \frac{(\alpha-1)lu}{u-l}.
\end{align}
Existing works simply set
$\tilde{\alpha}=\alpha$~\cite{MelloukiIbnKhedherElYacoubi23}. We observe
that by selecting $\tilde{\alpha}$ depending on $l, u$ we can reduce the
local relaxation error which has been shown to improve verification
performance~\cite{Zhang+18,Singh+19a,HashemiKouvarosLomuscio21}. We
minimise the relaxation error using the following result:
\begin{theorem}
\label{th:optimal_leaky_relu_relaxation}
For $\text{LeakyReLU}(x) = \max \{\alpha x, x\}$ with $x \in [l, u]$, the
local relaxation error is minimised by setting
\begin{equation}
f_\text{lower}(x) = \begin{dcases}
\alpha x \quad &\text{if} \:\: u < |l|,\\
x \quad &\text{else}
\end{dcases}
\end{equation}
\end{theorem}
\begin{proof}
See Appendix~\ref{sec:proof_optimality_leaky_relu_relaxation}.
\end{proof}

\section{Evaluation}
\label{sec:evaluation}
We implemented \ioucert{} on top of {\sf Venus}, a state-of-the-art
verifier~\cite{KouvarosLomuscio21}, encoding
Algorithm~\ref{alg:verification} as a custom layer appended to the target
model. It takes concrete bounds on the output logits (from
back-substitution) and computes the IoU and confidence-score bounds,
integrating with {\sf Venus}'s branch-and-bound (BaB) procedure.

\subsection{Benchmarks}
While some OD verification benchmarks are available within the Verification
of Neural Networks Competition (VNN-COMP)~\cite{Brix+24}, their verification queries
only target specific anchor boxes and class predictions rather than
assessing the robustness of the entire OD pipeline. Prior work on OD
verification that did consider object localisation focused on toy models
rather than anchor-based architectures~\cite{Cohen+24,Raviv+24}. To address
these gaps, we train various object detection models and modify an existing
benchmark to evaluate our framework. Following standard practices in the formal verification community~\cite{Brix+24,Kaulen+25}, we evaluate our method on random subsets of 50 correctly classified images per dataset, balancing mathematical guarantees with computational feasibility.

\begin{itemize}
\item {\em SSD.} We trained an SSD model~\cite{Liu+16} on the safety-critical LARD runway detection task~\cite{Ducoffe+23} (\textit{Google Earth} images, each depicting a single runway). Images were resized to $128 \times 128$, a relatively high resolution for complete verification; for tractability we replaced piece-wise linear \textit{MaxPool} with linear \textit{AvgPool} layers, and trained with stochastic gradient descent (SGD) and the \textit{MultiBox} loss~\cite{Liu+16}. NMS used a threshold of $0.5$ and a confidence threshold of $0.15$, with the highest-scoring box selected at inference.
Overall, the model contains $11,278,046$ learnable parameters.
\item {\em YOLOv2.} We used the YOLOv2-tiny (TinyYOLO)
benchmark from VNN-COMP 2023\footnote{https://github.com/xiangruzh/Yolo-Benchmark} which consists of a simplified YOLOv2-tiny model trained on a subset of images from the Pascal VOC dataset \cite{Everingham+10}. The authors
replaced the large backbone with a smaller variant, but the anchor-based
prediction head remains representative of foundational detection architectures.
\item {\em YOLOv3.} We trained YOLOv3-tiny models on LARD at $64 \times 64$ and $128 \times 128$ and on COCO at $128 \times 128$, again using AvgPool layers for tractability. As our framework focuses on the single-object scenario, we preprocessed COCO into single-object crops. The resulting models have between $8.7$ and $8.9$ million parameters; see Appendix~\ref{ssec:additional_details_models} for preprocessing and training details.
\end{itemize}

\noindent\textbf{Impact of Model Adaptations.} The
adaptations required for tractable verification only mildly affect standard
performance: for YOLOv3-tiny on LARD ($64{\times}64$), replacing MaxPool with
AvgPool moves $\text{mAP}_{0.5}$ from $86.88\%$ to $86.59\%$
($\text{mAP}_{0.5:0.95}$: $41.67\%\!\to\!40.90\%$) while reducing
verification time by over an order of magnitude. The backbones, resolution,
and single-object COCO crops all preserve the anchor-based structure our
method targets. Full before/after accuracies and the pooling ablation are
provided in Appendices~\ref{ssec:additional_details_models}
and~\ref{ssec:ablation_study_pooling}.

\subsection{Experimental Results}
We evaluate the effectiveness of \ioucert{} using a number of different
models and perturbations. To compare the tightness of our bounds as well as
the complete verification performance against the state-of-the-art, we
reimplement the bounding method proposed by Cohen et al.~\cite{Cohen+24} in our
verification framework. All experiments were run on a machine equipped with
an AMD Ryzen 9 9950X3D 16-core CPU, 192~GB of RAM, and an NVIDIA RTX~5090
GPU with 32~GB of VRAM, running Ubuntu with kernel 6.8.

\begin{table}[!ht]
\centering
\setlength{\tabcolsep}{2.0mm}
\caption{\ioucert{} verification results on SSD (trained on LARD) and YOLOv2 (trained on Pascal VOC) for different perturbation sizes $\epsilon$ and perturbations showing robust (R), non-robust (NR), and timeout (T) counts and mean time (all cases) in seconds.}
\label{tab:ssd_yolov2_results}
\begin{tabular}{lccccccccc}
\toprule
\multirow{2}{*}{\textbf{Model}} & \multirow{2}{*}{\textbf{$\epsilon$}} & \multicolumn{4}{c}{Brightness} & \multicolumn{4}{c}{Contrast} \\
\cmidrule(lr){3-6} \cmidrule(lr){7-10}
 & & \textbf{R} & \textbf{NR} & \textbf{T} & \textbf{Time} & \textbf{R} & \textbf{NR} & \textbf{T} & \textbf{Time} \\
\midrule
\multirow{7}{*}{SSD} & 0.01 & 48 & 2 & 0 & 29.06 & 49 & 1 & 0 & 24.21 \\
 & 0.05 & 45 & 5 & 0 & 404.54 & 47 & 3 & 0 & 171.49 \\
 & 0.10 & 40 & 10 & 0 & 731.41 & 42 & 8 & 0 & 359.41 \\
 & 0.30 & 9 & 41 & 0 & 458.21 & 30 & 20 & 0 & 885.51 \\
 & 0.50 & 0 & 47 & 3 & 221.71 & 14 & 36 & 0 & 743.04 \\
 & 0.80 & 0 & 50 & 0 & 10.62 & 0 & 50 & 0 & 3.78 \\
 & 1.00 & 0 & 50 & 0 & 5.79 & 0 & 50 & 0 & 3.90 \\
\midrule
\multirow{7}{*}{YOLOv2} & 0.01 & 50 & 0 & 0 & 3.69 & 50 & 0 & 0 & 3.23 \\
 & 0.05 & 50 & 0 & 0 & 12.97 & 50 & 0 & 0 & 4.86 \\
 & 0.10 & 47 & 3 & 0 & 23.81 & 50 & 0 & 0 & 9.23 \\
 & 0.30 & 28 & 22 & 0 & 39.40 & 48 & 2 & 0 & 33.36 \\
 & 0.50 & 4 & 46 & 0 & 9.99 & 36 & 14 & 0 & 35.86 \\
 & 0.80 & 0 & 50 & 0 & 1.21 & 0 & 50 & 0 & 2.52 \\
 & 1.00 & 0 & 50 & 0 & 0.81 & 0 & 50 & 0 & 2.36 \\
\bottomrule
\end{tabular}
\end{table}

\noindent\textbf{SSD and YOLOv2 Results.} We evaluated the ReLU-based SSD
and YOLOv2 models under both brightness and contrast perturbations.
Table~\ref{tab:ssd_yolov2_results} reports the number of \textit{ROBUST},
\textit{TIMEOUT}, and \textit{NONROBUST} cases as well as the average
verification time. \ioucert{} is fast for the small YOLOv2 model. It
verifies all properties for small perturbation budgets and most
properties for medium-sized budgets. For larger budgets, \ioucert{}
effectively identifies counterexamples showcasing the vulnerabilities of
the model. Although the SSD model is significantly larger, we are
able to identify robust cases for $\epsilon$ values of up to $0.3$ for
brightness and $0.5$ for contrast perturbations. As expected, verification
times are higher than for the small YOLOv2 model, because bound-propagation
passes through the larger model are more expensive and more branching is
needed for tight bounds. Comparing our bounds
(Section~\ref{ssec:optimal_iou_ibp_bounds}) with a reimplementation of the
looser method of Cohen et al.~\cite{Cohen+24}, performance is similar:
tighter bounds avoid branching but cost more to compute (detailed analysis
in Appendix~\ref{sec:complete_verification_performance}).

\begin{table}[!htb]
\centering
\caption{Summary of IoU bound ranges, number of sampled bounds, average tightness improvement, and avoided subproblem exploration percentages.}
\label{tab:tightness}
\begin{tabular}{cccc}
\toprule
\textbf{Range} & \textbf{\#Bounds} & \textbf{Improv. (\%)} & \textbf{Avoided (\%)} \\
\midrule
0.01 - 0.10 & 14642 & 50.67 & \;0.59 \\
0.10 - 0.20 & 8479 & 65.09 & \;0.12 \\
0.20 - 0.30 & 8084 & 58.54 & \;0.11 \\
0.30 - 0.40 & 6383 & 56.09 & \;0.05 \\
0.40 - 0.50 & 4440 & 55.32 & \;0.14 \\
0.50 - 0.60 & 3569 & 54.75 & 99.66 \\
0.60 - 0.70 & 2707 & 53.74 & 98.93 \\
0.70 - 0.80 & 2336 & 53.14 & 97.60 \\
0.80 - 0.90 & 1802 & 52.20 & 96.50 \\
0.90 - 0.99 & 1374 & 52.30 & 95.92 \\
\bottomrule
\end{tabular}
\end{table}

\noindent\textbf{Bound Tightness.}
To assess the bound tightness, we recorded bounds for all boxes (not just
the top-scoring one) during verification runs on the SSD model under a
brightness perturbation with $\epsilon=0.02$ and measure the difference
between the upper and lower bounds.
Table~\ref{tab:tightness} summarises the results. The first column shows
the range of the recorded IoU bounds; the second, the number of sampled
bounds; the third, the percentage tightness improvement
over~\cite{Cohen+24}; and the fourth, the percentage of branches whose
exploration was avoided due to tighter bounds. Our method consistently
improved bound tightness by over 50\% across all depths. At shallower
depths, where bounds are generally looser, this translated to over 95\% of
branches being pruned from the verification process.

\noindent\textbf{YOLOv3 Results.}
\begin{table}[t]
\centering
\setlength{\tabcolsep}{1.2mm}
\caption{\ioucert{} results on YOLOv3 for different perturbation sizes $\epsilon$ and perturbations with robust (R), non-robust (NR), and timeout (T) counts shown alongside mean verification time in seconds.}
\label{tab:yolov3_results}
\begin{tabular}{lccccccccccccc}
\toprule
\multirow{2}{*}{\textbf{Model}} & \multirow{2}{*}{\textbf{$\epsilon$}} & \multicolumn{4}{c}{Brightness} & \multicolumn{4}{c}{Contrast} & \multicolumn{4}{c}{Motion Blur ($0^\circ$)} \\
\cmidrule(lr){3-6} \cmidrule(lr){7-10} \cmidrule(lr){11-14}
 &  & \textbf{R} & \textbf{NR} & \textbf{T} & \textbf{Time} & \textbf{R} & \textbf{NR} & \textbf{T} & \textbf{Time} & \textbf{R} & \textbf{NR} & \textbf{T} & \textbf{Time} \\
\midrule
\multirow{7}{*}{\makecell{LARD\\ $64 \times 64$}} & 0.01 & 50 & 0 & 0 & 3.35 & 50 & 0 & 0 & 3.28 & 50 & 0 & 0 & 3.15 \\
 & 0.05 & 50 & 0 & 0 & 9.72 & 50 & 0 & 0 & 8.06 & 50 & 0 & 0 & 3.30 \\
 & 0.10 & 50 & 0 & 0 & 20.45 & 50 & 0 & 0 & 15.82 & 50 & 0 & 0 & 3.83 \\
 & 0.30 & 50 & 0 & 0 & 56.26 & 50 & 0 & 0 & 39.25 & 50 & 0 & 0 & 16.89 \\
 & 0.50 & 47 & 3 & 0 & 89.58 & 49 & 1 & 0 & 58.63 & 50 & 0 & 0 & 31.91 \\
 & 0.80 & 36 & 14 & 0 & 105.72 & 49 & 1 & 0 & 114.87 & 49 & 1 & 0 & 68.61 \\
 & 1.00 & 28 & 22 & 0 & 103.87 & 0 & 50 & 0 & 43.65 & 45 & 5 & 0 & 91.45 \\
\midrule
\multirow{7}{*}{\makecell{LARD\\ $128 \times 128$}} & 0.01 & 50 & 0 & 0 & 10.32 & 50 & 0 & 0 & 9.53 & 50 & 0 & 0 & 6.96 \\
 & 0.05 & 50 & 0 & 0 & 107.77 & 50 & 0 & 0 & 100.64 & 50 & 0 & 0 & 8.44 \\
 & 0.10 & 50 & 0 & 0 & 190.99 & 50 & 0 & 0 & 168.27 & 50 & 0 & 0 & 10.76 \\
 & 0.30 & 42 & 8 & 0 & 381.03 & 43 & 7 & 0 & 349.66 & 50 & 0 & 0 & 74.52 \\
 & 0.50 & 40 & 10 & 0 & 592.56 & 40 & 10 & 0 & 428.93 & 50 & 0 & 0 & 152.65 \\
 & 0.80 & 23 & 27 & 0 & 534.56 & 35 & 15 & 0 & 571.47 & 50 & 0 & 0 & 320.19 \\
 & 1.00 & 12 & 38 & 0 & 304.45 & 0 & 50 & 0 & 136.48 & 49 & 1 & 0 & 433.41 \\
\midrule
\multirow{7}{*}{\makecell{COCO\\ $128 \times 128$}} & 0.01 & 50 & 0 & 0 & 8.93 & 50 & 0 & 0 & 7.90 & 50 & 0 & 0 & 7.38 \\
 & 0.05 & 47 & 3 & 0 & 56.09 & 50 & 0 & 0 & 22.08 & 50 & 0 & 0 & 9.19 \\
 & 0.10 & 46 & 4 & 0 & 120.60 & 50 & 0 & 0 & 61.46 & 50 & 0 & 0 & 14.55 \\
 & 0.30 & 36 & 14 & 0 & 272.45 & 45 & 5 & 0 & 194.66 & 49 & 1 & 0 & 85.77 \\
 & 0.50 & 29 & 21 & 0 & 376.78 & 43 & 7 & 0 & 314.13 & 48 & 2 & 0 & 167.99 \\
 & 0.80 & 17 & 33 & 0 & 355.23 & 31 & 19 & 0 & 450.79 & 40 & 10 & 0 & 289.62 \\
 & 1.00 & 6 & 44 & 0 & 169.74 & 0 & 50 & 0 & 79.65 & 38 & 12 & 0 & 401.69 \\
\bottomrule
\end{tabular}
\end{table}
Table~\ref{tab:yolov3_results} reports results on the LeakyReLU-based YOLOv3
architecture under brightness, contrast, and motion blur (kernel size 5)
perturbations. Thanks to its coordinate transformation, \ioucert{} retains
enough tightness to verify YOLOv3 across a wide range of perturbations. All
models are highly robust to $0^\circ$ motion blur, though our procedure
still finds edge cases yielding incorrect predictions at high budgets. The
same trend holds for other blur angles
(Appendix~\ref{ssec:motionblur_results_different_angles},
Table~\ref{tab:yolov3_more_motionblur_results}).

The model trained on the more complex COCO dataset is generally more
vulnerable to brightness and motion blur than LARD at the same $128 \times
128$ resolution (\eg $29$ vs.\ $40$ robust cases at $\epsilon=0.5$ under
brightness), though slightly more resilient to mid-range contrast. On LARD,
the $128 \times 128$ model attains higher clean accuracy than the $64 \times
64$ one but is more vulnerable to perturbations: higher input dimensionality
leads to looser bounds, so higher clean accuracy does not necessarily lead to higher
certified robustness. LeakyReLU relaxation tightness is discussed in
Appendix~\ref{sec:proof_optimality_leaky_relu_relaxation}.

\noindent\textbf{Pooling Choice.} Our YOLOv3 models replace the MaxPool
downsampling layers of the original architecture with AvgPool layers, which
are linear and can be represented exactly in the bound-propagation framework
while retaining comparable clean accuracy ($\text{mAP}_{0.5}$ of $86.59\%$
vs.\ $86.88\%$ on LARD at $64 \times 64$). This choice is decisive for
verifiability: on a $50$-image subset, the AvgPool model is verified more
than an order of magnitude faster than its MaxPool counterpart and incurs far
fewer timeouts (\eg under brightness at $\epsilon=0.3$, $50/50$ robust cases
in $56$\,s for AvgPool versus $15/50$ with $33$ timeouts and over
$1600$\,s for MaxPool). The full ablation across all perturbations is reported
in Appendix~\ref{ssec:ablation_study_pooling},
Table~\ref{tab:ablation_study_pooling}.

\noindent\textbf{Discussion.} Overall, \ioucert{} effectively verifies a
range of anchor-based detectors via our coordinate transformation and, for
YOLOv3, the tight LeakyReLU relaxations, scaling even to complex multi-class
datasets like COCO without sacrificing bound tightness. Verification is
effective regardless of the bounding method: tighter bounds cost more per
call but prune more branches, while looser bounds branch more but process
each branch faster (Appendix~\ref{sec:complete_verification_performance}).

\section{Scope and Limitations}
\label{sec:scope_and_limitations}
The coordinate transformation (Section~\ref{ssec:coordinate_transformation})
and optimal IoU bounds (Section~\ref{ssec:optimal_iou_ibp_bounds}) at the
core of \ioucert{} apply to any detector whose box-decoding map $\psi \circ
\phi$ is injective with a tractable inverse, \ie strictly monotonic in each
offset (Appendix~\ref{sec:h_phi_functions}). This holds for the dense
anchor-based heads of the SSD and YOLO families we evaluate, and is
independent of training: label-assignment strategies such as
ATSS~\cite{Zhang+20b}, PAA~\cite{KimLee20} or OTA~\cite{Ge+21} alter the
training target, not the inference-time decoding map. The same principle
covers the region proposal network of two-stage detectors such as Faster
R-CNN and anchor-free heads that regress invertible offsets, whereas
transformer-based detectors such as DETR add attention and set-prediction
mechanisms that remain challenging for current verifiers
(Appendix~\ref{sec:scope_and_extensions}).

\ioucert{} performs verification offline and currently targets anchor-based detectors for the single-object case, where correctness depends only on the IoU between the
prediction and the ground truth. Extending it to the full multi-object
pipeline additionally requires bounding the pairwise overlaps used by
non-maximum suppression (NMS), \ie the bounds
$\underline{J}_{ik} \le \mathrm{IoU}(B_i, B_k) \le \overline{J}_{ik}$ for
candidate pairs $i,k$. This is substantially harder, requiring up to
$O(n^2)$ certificates over \emph{pairs} of variable boxes and becoming
ambiguous whenever the overlap bounds straddle the NMS threshold. We therefore
view NMS-aware verification over the candidate boxes already bounded tightly
by \ioucert{} as the most promising next step, and discuss it further in
Appendix~\ref{sec:scope_and_extensions}.

\section{Conclusion}
\label{sec:conclusion}
Verifying object detectors before deployment matters in safety-critical
settings such as autonomous driving, yet their complex architectures and
non-linear localisation place realistic anchor-based detectors such as
YOLOv3 beyond existing robustness verification methods.

We introduced \ioucert{}, which combines optimal Interval Bound Propagation
(IBP) bounds for the Intersection-over-Union (IoU) metric with a coordinate
transformation for the box prediction function and tight LeakyReLU
relaxations, enabling the analysis of complex anchor-based detectors such as
YOLOv3 across diverse datasets at a scale not previously demonstrated.

\ioucert{} performs verification offline, prior to deployment, rather than
at runtime. Its current scope, its applicability to other detector families,
and the path towards full multi-object, NMS-aware verification are discussed
in Section~\ref{sec:scope_and_limitations}.

\section*{Acknowledgements}
Benedikt Br\"uckner acknowledges support from the UKRI Centre for Doctoral Training in Safe and Trusted Artificial Intelligence [EP/S023356/1]. Alejandro Mercado acknowledges support from an Imperial College London President's PhD Scholarship. Alessio Lomuscio acknowledges partial support from the Royal Academy of Engineering via a Chair of Emerging Technologies.

%
%
\bibliographystyle{splncs04}
\bibliography{bib}

\clearpage
\setcounter{page}{1}
\setcounter{section}{0}
\renewcommand{\thesection}{\Alph{section}}

\section{Proof: Without Loss of Generality on Single Highest-Confidence Box}
\label{sec:proof_single_highest_box}

We restate the assumption made in the paper:  
\textit{Without loss of generality, we assume inference outputs only the
highest-confidence bounding box above a preset threshold, which does not
affect the model's original performance.}

We now provide the formal justification. We consider a general version of OD
correctness, in which the output $O(I) = D = \{b_1, \ldots, b_k\}$ of
an object detection (OD) model $O$ on input $I$ is said to be correct with
respect to the ground truth $G = \{g_1, \ldots, g_m\}$ and a threshold
$\tau_{\mathrm{iou}}$ if:
\begin{enumerate}
\item $|D| = m$, and
\item For each $g \in G$, there exists a $b \in D$ such that $\argmax_{i
\in 1, \dots, n_c} \mathbf{c}_i = g_c$ with its class probability exceeding
$\tau_{\mathrm{class}}$ and $\mathrm{IoU}(b, g) \geq \tau_{\mathrm{iou}}$.
\end{enumerate}

In the single-object case ($m = 1$), correctness reduces to having just one
predicted box of the correct class matching the ground truth with
sufficient $\mathrm{IoU}$ and score.

Recall that the OD pipeline includes a postprocessing stage $\mathcal{P}$
(\eg Non-Maximum Suppression, confidence filtering) that transforms the
raw model output into the final set $D$.
We note:
\begin{itemize}
\item All postprocessing schemes $\mathcal{P}$ are designed to retain at
least the highest-scoring box above threshold, since this box represents
the most confident prediction.
\item Thus, among all $\mathcal{P}$, the subset of outputs always includes
the maximal-score box (or no box if all scores fall below threshold).
\end{itemize}

Thus, for the single-object case:
\begin{itemize}
\item If $\mathcal{P}$ yields a correct result, it necessarily has just one
correct box, and since the highest-scoring box is always part of the final
output, using only the highest-scoring box suffices to recover the same
correctness outcome.
\item If $\mathcal{P}$ yields an incorrect result, this either (i) includes
only one box that is not correct with respect to the ground truth (and for which just
keeping the highest-scoring one would also fail), or (ii) includes multiple
boxes, and thus keeping the highest-scoring one could either also fail, but
also could end up being correct.
\end{itemize}

Therefore, retaining only the highest-scoring bounding box is at least as
accurate as any other postprocessing method. We conclude that, without loss
of generality, we can reduce the postprocessing to selecting only the
highest-confidence bounding box above the threshold. This simplification
does not alter the model’s original correctness or performance under the
definition provided, and it simplifies the verification framework.

\section{Definition and Injectivity of the $\psi$ and $\phi$ Functions}
\label{sec:h_phi_functions}
In this section we illustrate the $\phi$ functions that the different object
detectors we analyse employ to convert their predicted logits to bounding
box predictions in the centre format. We
further show that these functions are strictly monotonic and in turn
injective/one-to-one which is a requirement for our coordinate
transformation. Besides this, we prove the injectivity of the $\psi$ function.

We first analyse the injectivity of a number of components that are often
employed as a part of $\phi$:
\begin{itemize}
\item Let $d_1: \mathbb{R} \to \mathbb{R}, x \mapsto e^x$, then
$\frac{\partial d_1(x)}{\partial x} = e^x > 0$, \ie $e^x$ is strictly
monotonically increasing and therefore injective.
\item Let $d_2: \mathbb{R}_{>0} \to \mathbb{R}_{>0}, x \mapsto x^2$, then
$\frac{\partial d_2(x)}{\partial x} = 2x > 0 \: \forall \: x>0$. Therefore
$d_2$ is strictly monotonically increasing on $\mathbb{R}_{>0}$ and
therefore injective.
\item Let $d_3: \mathbb{R} \to \mathbb{R}_{>0}, x \mapsto \sigma(x) :=
\frac{e^x}{1+e^{x}}$. We obtain that $\frac{\partial d_3(x)}{\partial x} =
\frac{e^x}{(1+e^x)^2}$. From $e^x>0$ it follows that $(1+e^x)^2>0$ and
therefore $\frac{e^x}{(1+e^x)^2}>0$, implying that $\sigma(x)$ is strictly
monotonically increasing and injective.
\item Let $d_4: \mathbb{R} \to \mathbb{R}, x \mapsto ax + b$ be an affine
function. $\frac{\partial d_4(x)}{\partial x} = a$, therefore $d_4$ is
strictly monotonically increasing for $a > 0$ and strictly monotonically
decreasing for $a < 0$. This implies that $d_4$ is injective for $a \neq 0$
\end{itemize}

We first analyse the $\psi$ function and restate its definition from Section
\ref{ssec:object_detection}:
\[
\psi(c_x, c_y, w, h) = (c_x - \tfrac{w}{2},\; c_y - \tfrac{h}{2},\; c_x + \tfrac{w}{2},\; c_y + \tfrac{h}{2})
\]
We observe that each component function $\psi_i: \mathbb{R} \to \mathbb{R}$ is
an affine function with slope $a = \pm \frac{1}{2} \neq 0$. From the
injectivity of $d_4$ it therefore follows that $\psi$ is injective.

To analyse the injectivity of the $\phi$ functions, we recall that the
composition of injective functions is injective~\cite[Theorem 12.2]{Hammack18}.

\subsection{SSD}
The $\phi$ functions for the SSD model are defined as
\begin{equation*}
\begin{aligned}
\phi_0(o_0, p_0, p_2) &= p_0 + o_0 \cdot \mathrm{var}_1 \cdot p_2 \quad &\text{(center $x$ coordinate)}, \\
\phi_1(o_1, p_1, p_3) &= p_1 + o_1 \cdot \mathrm{var}_1 \cdot p_3 \quad &\text{(center $y$ coordinate)}, \\
\phi_2(o_2, p_2) &= p_2 \cdot e^{o_2 \cdot \mathrm{var}_2} \quad &\text{(width)}, \\
\phi_3(o_3, p_3) &= p_3 \cdot e^{o_3 \cdot \mathrm{var}_2} \quad &\text{(height)}.
\end{aligned}
\end{equation*}
where $\mathrm{var}_1$ and $\mathrm{var}_2$ are pre-defined values. Assuming that
$\mathrm{var}_1, p_2, p_3 > 0$ it is obvious that $\phi_0, \phi_1$ are
affine and therefore injective. Assuming that $p_2, p_3 > 0$ we also
find that $\phi_2, \phi_3$ are injective as a composition of injective
functions.

\subsection{YOLOv2}
The $\phi$ functions for the YOLOv2 model are defined as
\begin{equation*}
\begin{aligned}
\phi_0(o_0, p_0, s) &= (\sigma(o_0) + p_0) \cdot s \quad &\text{(center $x$ coordinate)}, \\
\phi_1(o_1, p_1, s) &= (\sigma(o_1) + p_1) \cdot s \quad &\text{(center $y$ coordinate)}, \\
\phi_2(o_2, p_2, s) &= p_2 \cdot e^{o_2} \cdot s \quad &\text{(width)}, \\
\phi_3(o_3, p_3, s) &= p_3 \cdot e^{o_3} \cdot s \quad &\text{(height)},
\end{aligned}
\end{equation*}
where $s=\frac{\mathrm{image\_size}}{\mathrm{grid\_size}}$. By definition, it holds that $s >
0$. We further assume that $p_0, p_1, p_2, p_3 > 0$ which implies that,
as compositions of injective functions, all functions $\phi_i$ are
injective.

\subsection{YOLOv3}
The YOLOv3 architecture we use differs from YOLOv2 in that it employs
different functions for calculating the box predictions based on the raw
output logits of the model. This change is meant to avoid issues that
occurred in the original model when the box centre was located close to the
boundaries of a grid cell. Besides this, YOLOv3 employs multiple prediction
grids with different scales. The scale $s_k$ associated with the $k$-th
anchor box therefore depends on the size of the grid which predicted that
box.
\[
\begin{aligned}
\phi_0(o_0, p_0, s_k) &= (2 \cdot \sigma(o_0) - 0.5 + p_0) \cdot s_k \quad &\text{(center $x$ coordinate)}, \\
\phi_1(o_1, p_1, s_k) &= (2 \cdot \sigma(o_1) - 0.5 + p_1) \cdot s_k \quad &\text{(center $y$ coordinate)}, \\
\phi_2(o_2, p_2, s_k) &= \left ( 2 \cdot \sigma(o_2) \right )^2 \cdot p_2 \cdot s_k \quad &\text{(width)}, \\
\phi_3(o_3, p_3, s_k) &= \left ( 2 \cdot \sigma(o_3) \right )^2 \cdot p_3 \cdot s_k \quad &\text{(height)},
\end{aligned}
\]
By definition, we once again have that $s_k >0$ which directly implies that
$\phi_0, \phi_1$ are injective as compositions of injective functions.
For $\phi_2, \phi_3$ we observe that $\sigma(x) > 0 \: \forall \: x \in
\mathbb{R}$ which implies that $2 \cdot \sigma(o_2)$ and $2 \cdot
\sigma(o_3)$ are injective. Assuming $p_2, p_3 > 0$, it directly follows
that $\phi_2, \phi_3$ are injective as compositions of injective
functions.

\section{Further Details on the Experiments and Additional Results}
\label{sec:further_details_experiments}
\subsection{Additional Details on the Trained Models}
\label{ssec:additional_details_models}

\begin{table}[ht]
\centering
\caption{Optimal hyperparameters for the 64$\times$64 YOLOv3 training runs and resulting accuracy. $\text{mAP}_{0.5}$ denotes the Mean Average Precision at an IoU threshold of $0.5$ while $\text{mAP}_{0.5:0.95}$ denotes the Mean Average Precision averaged over IoU thresholds between $0.5$ and $0.95$ in increments of $0.05$.}
\label{tab:yolov3_64_parameters}
\begin{tabular}{lcc}
\toprule
& \textbf{\makecell{YOLOv3-tiny, \\ 64$\times$64}} & \textbf{\makecell{YOLOv3-tiny-maxpool, \\ 64$\times$64}} \\
\midrule
Dataset & LARD & LARD \\
Weight Decay & $10^{-3}$ & $10^{-3}$ \\
Learning Rate & $5 \cdot 10^{-3}$ & $5 \cdot 10^{-3}$ \\
Epochs & 400 & 800 \\
Left-Right Flipping Probability & 0.5 & 0.5 \\
Mosaic Probability & 0.5 & 0.5 \\
MixUp Probability & 0.5 & 0.5 \\
\midrule
$\text{mAP}_{0.5}$ & 86.59\% & 86.88\% \\
$\text{mAP}_{0.5:0.95}$ & 40.90\% & 41.67\% \\
\bottomrule
\end{tabular}
\end{table}

\begin{table}[ht]
\centering
\caption{Optimal hyperparameters for the 128$\times$128 YOLOv3 training runs and resulting accuracy. $\text{mAP}_{0.5}$ denotes the Mean Average Precision at an IoU threshold of $0.5$ while $\text{mAP}_{0.5:0.95}$ denotes the Mean Average Precision averaged over IoU thresholds between $0.5$ and $0.95$ in increments of $0.05$.}
\label{tab:yolov3_128_parameters}
\begin{tabular}{lcc}
\toprule
& \textbf{\makecell{YOLOv3-tiny, \\ 128$\times$128}} & \textbf{\makecell{YOLOv3-tiny, \\ 128$\times$128}} \\
\midrule
Dataset & LARD & COCO \\
Weight Decay & $10^{-4}$ & $5 \cdot 10^{-4}$ \\
Learning Rate & $5 \cdot 10^{-3}$ & $10^{-2}$ \\
Epochs & 2000 & 2000 \\
Left-Right Flipping Probability & 0.5 & 0.5 \\
Mosaic Probability & 0.5 & 0.5 \\
MixUp Probability & 0.5 & 0.5 \\
\midrule
$\text{mAP}_{0.5}$ & 98.78\% &  44.99\% \\
$\text{mAP}_{0.5:0.95}$ & 71.24\% & 26.56\% \\
\bottomrule
\end{tabular}
\end{table}

\begin{table}[ht]
\centering
\setlength{\tabcolsep}{1.2mm}
\caption{YOLOv3 results for motion blur perturbations with angles of $45,
90$ and $135$ degrees and different epsilon values. The numbers of robust (R),
non-robust (NR) and timeout (T) cases are shown together with the mean verification time
in seconds. The timeout was set to $1800$ seconds.}
\label{tab:yolov3_more_motionblur_results}
\begin{tabular}{lccccccccccccc}
\toprule
\multirow{2}{*}{\textbf{Model}} & \multirow{2}{*}{\textbf{$\epsilon$}} & \multicolumn{4}{c}{Motion Blur ($45^\circ$)} & \multicolumn{4}{c}{Motion Blur ($90^\circ$)} & \multicolumn{4}{c}{Motion Blur ($135^\circ$)} \\
\cmidrule(lr){3-6} \cmidrule(lr){7-10} \cmidrule(lr){11-14}
 &  & \textbf{R} & \textbf{NR} & \textbf{T} & \textbf{Time} & \textbf{R} & \textbf{NR} & \textbf{T} & \textbf{Time} & \textbf{R} & \textbf{NR} & \textbf{T} & \textbf{Time} \\
\midrule
\multirow{7}{*}{\makecell{LARD\\ $64 \times 64$}} & 0.01 & 50 & 0 & 0 & 3.15 & 50 & 0 & 0 & 3.15 & 50 & 0 & 0 & 3.16 \\
 & 0.05 & 50 & 0 & 0 & 3.31 & 50 & 0 & 0 & 3.30 & 50 & 0 & 0 & 3.32 \\
 & 0.10 & 50 & 0 & 0 & 3.90 & 50 & 0 & 0 & 3.72 & 50 & 0 & 0 & 4.00 \\
 & 0.30 & 50 & 0 & 0 & 17.97 & 50 & 0 & 0 & 15.42 & 50 & 0 & 0 & 18.28 \\
 & 0.50 & 50 & 0 & 0 & 34.83 & 49 & 1 & 0 & 28.24 & 49 & 1 & 0 & 33.59 \\
 & 0.80 & 48 & 2 & 0 & 71.10 & 48 & 2 & 0 & 58.78 & 46 & 4 & 0 & 69.32 \\
 & 1.00 & 47 & 3 & 0 & 100.88 & 46 & 4 & 0 & 87.15 & 43 & 7 & 0 & 91.41 \\
\midrule
\multirow{7}{*}{\makecell{LARD\\ $128 \times 128$}} & 0.01 & 50 & 0 & 0 & 6.70 & 50 & 0 & 0 & 6.33 & 50 & 0 & 0 & 5.76 \\
 & 0.05 & 50 & 0 & 0 & 7.96 & 50 & 0 & 0 & 7.62 & 50 & 0 & 0 & 7.27 \\
 & 0.10 & 50 & 0 & 0 & 11.54 & 50 & 0 & 0 & 9.87 & 50 & 0 & 0 & 10.53 \\
 & 0.30 & 49 & 1 & 0 & 78.03 & 50 & 0 & 0 & 68.96 & 50 & 0 & 0 & 77.64 \\
 & 0.50 & 49 & 1 & 0 & 159.30 & 50 & 0 & 0 & 140.51 & 50 & 0 & 0 & 162.31 \\
 & 0.80 & 49 & 1 & 0 & 324.29 & 45 & 5 & 0 & 252.23 & 48 & 2 & 0 & 326.73 \\
 & 1.00 & 48 & 2 & 0 & 450.14 & 44 & 6 & 0 & 368.42 & 46 & 4 & 0 & 451.50 \\
\midrule
\multirow{7}{*}{\makecell{COCO\\ $128 \times 128$}} & 0.01 & 50 & 0 & 0 & 7.53 & 50 & 0 & 0 & 7.09 & 50 & 0 & 0 & 7.31 \\
 & 0.05 & 50 & 0 & 0 & 9.16 & 50 & 0 & 0 & 8.99 & 50 & 0 & 0 & 8.79 \\
 & 0.10 & 50 & 0 & 0 & 16.85 & 49 & 1 & 0 & 13.81 & 50 & 0 & 0 & 16.00 \\
 & 0.30 & 48 & 2 & 0 & 98.59 & 49 & 1 & 0 & 85.49 & 49 & 1 & 0 & 98.71 \\
 & 0.50 & 46 & 4 & 0 & 193.46 & 48 & 2 & 0 & 172.18 & 45 & 5 & 0 & 185.92 \\
 & 0.80 & 33 & 17 & 0 & 268.18 & 33 & 17 & 0 & 236.34 & 37 & 13 & 0 & 293.91 \\
 & 1.00 & 30 & 20 & 0 & 356.31 & 31 & 19 & 0 & 327.65 & 31 & 19 & 0 & 352.60 \\
\bottomrule
\end{tabular}
\end{table}
We provide additional information on the self-trained models in this
section. The YOLOv3 models are YOLOv3-tiny models which follow the
architecture described in the original paper~\cite{RedmonFarhadi18}, we
make use of the Ultralytics YOLOv3 repository to train these
models~\cite{YOLOv3Ultralytics}. We adapt the models to the verification
task by replacing the MaxPool with AvgPool layers. A more detailed
discussion and evaluation on this can be found in Appendix
\ref{ssec:ablation_study_pooling}. We employ an IoU threshold
$\tau_{\mathrm{iou}}=0.5$ and a class threshold of $0.15$ during training.

For the LARD dataset, we train models at a $64 \times 64$ and a $128
\times 128$ resolution. Since $64 \times 64$ is a relatively small
resolution for an object detection task and the LARD dataset contains a
large number of images where the runway is far away from the plane, we need
to crop the images such that the runway is at least $10$ pixels in size.
The YOLOv3-tiny model on the $64 \times 64$ inputs has $8,666,692$
trainable parameters while the $128 \times 128$ model has $8,669,876$
trainable parameters. Both models possess two prediction heads at different
scales, one which operates at a stride of $16$ and one which operates at a
stride of $32$. We use the autoanchor functionality to find suitable anchor
boxes for the LARD dataset. For the $64 \times 64$ model we use $(13, 13),
(21, 14), (15, 22)$ as the anchor boxes for the head with a stride of $16$
and $(33, 20), (22, 33), (43, 39)$ as the anchors for the head with a
stride of $32$. For the model trained at the higher $128 \times 128$
resolution we obtain $(13, 13), (15, 23), (25, 17)$ as the boxes for the
head with stride $16$ and $(27, 32), (50, 35), (80, 59)$ for the head with
a stride of $32$.

For the COCO dataset, we train a model at a $128
\times 128$ resolution. We preprocess the dataset to produce crops which only contain a single object since we focus on single-object detection. This may result in multiple images being generated from a single base image if it contains multiple separable objects. In the case of objects which overlap with others and therefore cannot be separated, we discard the object and move on to the next one to attempt cropping there.
Since $128 \times 128$ is a small image size compared to the original size of the images, we crop the images such that the object is at least $10$ pixels in size. The YOLOv3-tiny model on the $128 \times 128$ inputs has $8,852,366$
trainable parameters. The model has two prediction heads at different
scales, one which operates at a stride of $16$ and one which operates at a
stride of $32$. We use the autoanchor functionality to find suitable anchor
boxes for our input size and find that using boxes shaped $(11, 11), (16, 27), (39, 36)$ is optimal for the
head with stride $16$ and $(49, 97), (97, 53), (104, 103)$ are the optimal shapes for the head with
a stride of $32$.

All models are trained using a Stochastic Gradient Descent optimiser with
an initial learning rate of $0.01$ which is decayed using a Cosine
Annealing learning rate scheduler. We tune the number of epochs, the data
augmentation strategies, the learning rate and the weight decay during our
experiments. The optimal hyperparameters and the performance that we obtain
are shown in Table~\ref{tab:yolov3_64_parameters} for the models trained at a $64 \times 64$ resolution and in Table~\ref{tab:yolov3_128_parameters} for those trained at a $128 \times 128$ resolution. The test
accuracies which are provided are obtained on the synthetic LARD test
dataset for the LARD models and the COCO validation dataset for the COCO model.

\subsection{Motion Blur Perturbations With Different Angles on YOLOv3}
\label{ssec:motionblur_results_different_angles}
For the sake of completeness, we present the verification results on
YOLOv3-tiny models using motion blur perturbations with varying blurring
angles in Table~\ref{tab:yolov3_more_motionblur_results}. We generally
observe similar tendencies in terms of certified robustness and
verification times across the different blurring angles as we do for the
$0^\circ$ angle presented in the main text. Notably, the model trained on
the COCO dataset consistently exhibits a steeper drop in certified robustness
at higher perturbation budgets ($\epsilon \geq 0.80$) across all blur angles
compared to the LARD models, confirming its higher vulnerability to this
specific perturbation compared to the LARD dataset.

\subsection{Ablation Study on MaxPool vs. AvgPool Pooling}
\label{ssec:ablation_study_pooling}
\begin{table}[!ht]
\centering
\setlength{\tabcolsep}{1.8mm}
\caption{Ablation Study comparing YOLOv3-tiny models with MaxPool vs. AvgPool pooling layers trained on the LARD dataset at a 64$\times$64 resolution. We show the number of robust (R), non-robust (NR) and timeout (T) cases together with the mean verification time in seconds.}
\label{tab:ablation_study_pooling}
\adjustbox{max width=0.82\textwidth, keepaspectratio}{
\begin{tabular}{lccccccccc}
\toprule
\multirow{2}{*}{\textbf{Perturbation}} & \multirow{2}{*}{\textbf{$\epsilon$}} & \multicolumn{4}{c}{AvgPool} & \multicolumn{4}{c}{MaxPool}\\
\cmidrule(lr){3-6} \cmidrule(lr){7-10}
 & & \textbf{R} & \textbf{NR} & \textbf{T} & \textbf{Time} & \textbf{R} & \textbf{NR} & \textbf{T} & \textbf{Time} \\
\midrule
\multirow{7}{*}{Brightness} & 0.01 & 50 & 0 & 0 & 3.35 & 50 & 0 & 0 & 112.67 \\
 & 0.05 & 50 & 0 & 0 & 9.72 & 48 & 0 & 2 & 556.30 \\
 & 0.10 & 50 & 0 & 0 & 20.45 & 41 & 0 & 9 & 969.82 \\
 & 0.30 & 50 & 0 & 0 & 56.26 & 15 & 2 & 33 & 1602.32 \\
 & 0.50 & 47 & 3 & 0 & 89.58 & 0 & 5 & 45 & 1620.41 \\
 & 0.80 & 36 & 14 & 0 & 105.72 & 0 & 16 & 34 & 1225.74 \\
 & 1.00 & 28 & 22 & 0 & 103.87 & 0 & 23 & 27 & 974.65 \\
\midrule
\multirow{7}{*}{Contrast} & 0.01 & 50 & 0 & 0 & 3.28 & 50 & 0 & 0 & 215.24 \\
 & 0.05 & 50 & 0 & 0 & 8.06 & 44 & 0 & 6 & 866.93 \\
 & 0.10 & 50 & 0 & 0 & 15.82 & 30 & 0 & 20 & 1229.32 \\
 & 0.30 & 50 & 0 & 0 & 39.25 & 6 & 0 & 44 & 1750.99 \\
 & 0.50 & 49 & 1 & 0 & 58.63 & 0 & 1 & 49 & 1764.01 \\
 & 0.80 & 49 & 1 & 0 & 114.87 & 0 & 4 & 46 & 1656.06 \\
 & 1.00 & 0 & 50 & 0 & 43.65 & 0 & 45 & 5 & 182.96 \\
\midrule
\multirow{7}{*}{Motion Blur $0^\circ$} & 0.01 & 50 & 0 & 0 & 3.15 & 50 & 0 & 0 & 12.56 \\
 & 0.05 & 50 & 0 & 0 & 3.30 & 50 & 0 & 0 & 53.53 \\
 & 0.10 & 50 & 0 & 0 & 3.83 & 50 & 0 & 0 & 110.31 \\
 & 0.30 & 50 & 0 & 0 & 16.89 & 50 & 0 & 0 & 347.36 \\
 & 0.50 & 50 & 0 & 0 & 31.91 & 49 & 0 & 1 & 636.65 \\
 & 0.80 & 49 & 1 & 0 & 68.61 & 44 & 1 & 5 & 1150.29 \\
 & 1.00 & 45 & 5 & 0 & 91.45 & 29 & 1 & 20 & 1454.13 \\
\midrule
\multirow{7}{*}{Motion Blur $45^\circ$} & 0.01 & 50 & 0 & 0 & 3.15 & 50 & 0 & 0 & 13.34 \\
 & 0.05 & 50 & 0 & 0 & 3.31 & 50 & 0 & 0 & 58.30 \\
 & 0.10 & 50 & 0 & 0 & 3.90 & 50 & 0 & 0 & 118.77 \\
 & 0.30 & 50 & 0 & 0 & 17.97 & 50 & 0 & 0 & 384.12 \\
 & 0.50 & 50 & 0 & 0 & 34.83 & 49 & 0 & 1 & 713.62 \\
 & 0.80 & 48 & 2 & 0 & 71.10 & 38 & 4 & 8 & 1195.97 \\
 & 1.00 & 47 & 3 & 0 & 100.88 & 21 & 10 & 19 & 1241.84 \\
\midrule
\multirow{7}{*}{Motion Blur $90^\circ$} & 0.01 & 50 & 0 & 0 & 3.15 & 50 & 0 & 0 & 8.34 \\
 & 0.05 & 50 & 0 & 0 & 3.30 & 50 & 0 & 0 & 36.50 \\
 & 0.10 & 50 & 0 & 0 & 3.72 & 50 & 0 & 0 & 73.40 \\
 & 0.30 & 50 & 0 & 0 & 15.42 & 49 & 1 & 0 & 224.68 \\
 & 0.50 & 49 & 1 & 0 & 28.24 & 49 & 1 & 0 & 420.46 \\
 & 0.80 & 48 & 2 & 0 & 58.78 & 44 & 5 & 1 & 772.38 \\
 & 1.00 & 46 & 4 & 0 & 87.15 & 39 & 8 & 3 & 955.88 \\
\midrule
\multirow{7}{*}{Motion Blur $135^\circ$} & 0.01 & 50 & 0 & 0 & 3.16 & 50 & 0 & 0 & 13.61 \\
 & 0.05 & 50 & 0 & 0 & 3.32 & 50 & 0 & 0 & 58.36 \\
 & 0.10 & 50 & 0 & 0 & 4.00 & 50 & 0 & 0 & 117.14 \\
 & 0.30 & 50 & 0 & 0 & 18.28 & 50 & 0 & 0 & 368.98 \\
 & 0.50 & 49 & 1 & 0 & 33.59 & 48 & 0 & 2 & 731.39 \\
 & 0.80 & 46 & 4 & 0 & 69.32 & 43 & 3 & 4 & 1210.54 \\
 & 1.00 & 43 & 7 & 0 & 91.41 & 27 & 6 & 17 & 1414.50 \\
\bottomrule
\end{tabular}
}
\end{table}

\begin{table}[!ht]
\centering
\caption{Comparison of verification performance between different bounding
methods that have been integrated into our framework on SSD and YOLOv2. We
show the verification time, number of explored branches and branching depth
until resolution for all queries. The percentage gain of our method over that
of Cohen et al.~\cite{Cohen+24} in terms of time, \#branches and depth is denoted as
$\Delta_t$, $\Delta_b$ and $\Delta_d$, respectively. }
\label{tab:comparison_iou_bounding_methods}
\adjustbox{max width=\textwidth}{
\begin{tabular}{c c c r c c r c c r c c}
\toprule
 & & & \multicolumn{3}{c}{Ours} & \multicolumn{3}{c}{Cohen et al.~\cite{Cohen+24}} & \multicolumn{3}{c}{Gain (\%)} \\
\cmidrule(lr){4-6} \cmidrule(lr){7-9} \cmidrule(lr){10-12}
Model & Perturbation & $\epsilon$ & \multicolumn{1}{c}{Time} & \#Branches & Depth & \multicolumn{1}{c}{Time} & \#Branches & Depth & \multicolumn{1}{c}{$\Delta_t$} & \multicolumn{1}{c}{$\Delta_b$} & \multicolumn{1}{c}{$\Delta_d$} \\
\midrule
\multirow{16}{*}{SSD} & \multirow{8}{*}{Brightness} & 0.01 & 29.06 & 66 & 8 & 29.99 & 74 & 12 & 3.1 & 10.8 & 33.3 \\
 &  & 0.02 & 114.12 & 370 & 103 & 114.45 & 378 & 104 & 0.3 & 2.1 & 1.0 \\
 &  & 0.05 & 404.54 & 1192 & 179 & 389.59 & 1236 & 179 & -3.8 & 3.6 & 0.0 \\
 &  & 0.1 & 731.41 & 2027 & 201 & 738.81 & 2103 & 202 & 1.0 & 3.6 & 0.5 \\
 &  & 0.3 & 458.21 & 1208 & 57 & 458.51 & 1212 & 58 & 0.1 & 0.3 & 1.7 \\
 &  & 0.5 & 221.71 & 622 & 25 & 221.70 & 622 & 25 & -0.0 & 0.0 & 0.0 \\
 &  & 0.8 & 10.62 & 59 & 9 & 10.65 & 59 & 9 & - & - & - \\
 &  & 1 & 5.79 & 55 & 5 & 5.49 & 55 & 5 & - & - & - \\
\cmidrule(lr){2-12}
 & \multirow{8}{*}{Contrast} & 0.01 & 24.21 & 54 & 2 & 24.81 & 58 & 4 & 2.4 & 6.9 & 50.0 \\
 &  & 0.02 & 29.68 & 72 & 11 & 31.37 & 84 & 15 & 5.4 & 14.3 & 26.7 \\
 &  & 0.05 & 171.49 & 550 & 135 & 163.62 & 562 & 137 & -4.8 & 2.1 & 1.5 \\
 &  & 0.1 & 359.41 & 1055 & 171 & 361.96 & 1088 & 173 & 0.7 & 3.0 & 1.2 \\
 &  & 0.3 & 885.51 & 2324 & 185 & 887.91 & 2362 & 186 & 0.3 & 1.6 & 0.5 \\
 &  & 0.5 & 743.04 & 1904 & 114 & 742.45 & 1918 & 116 & -0.1 & 0.7 & 1.7 \\
 &  & 0.8 & 3.78 & 53 & 2 & 3.78 & 53 & 2 & - & - & - \\
 &  & 1 & 3.90 & 53 & 2 & 3.77 & 53 & 2 & - & - & - \\
\midrule
\multirow{16}{*}{YOLOv2} & \multirow{8}{*}{Brightness} & 0.01 & 3.69 & 58 & 4 & 3.65 & 58 & 4 & -1.1 & 0.0 & 0.0 \\
 &  & 0.02 & 5.15 & 94 & 18 & 5.11 & 94 & 18 & -0.8 & 0.0 & 0.0 \\
 &  & 0.05 & 12.97 & 308 & 79 & 12.94 & 312 & 81 & -0.2 & 1.3 & 2.5 \\
 &  & 0.1 & 23.81 & 627 & 135 & 23.79 & 639 & 137 & -0.1 & 1.9 & 1.5 \\
 &  & 0.3 & 39.40 & 1096 & 163 & 38.15 & 1102 & 163 & -3.3 & 0.5 & 0.0 \\
 &  & 0.5 & 9.99 & 284 & 41 & 9.81 & 284 & 41 & -1.8 & 0.0 & 0.0 \\
 &  & 0.8 & 1.21 & 55 & 4 & 1.22 & 55 & 4 & - & - & - \\
 &  & 1 & 0.81 & 50 & 0 & 0.82 & 50 & 0 & - & - & - \\
\cmidrule(lr){2-12}
 & \multirow{8}{*}{Contrast} & 0.01 & 3.23 & 52 & 1 & 3.23 & 52 & 1 & 0.0 & 0.0 & 0.0 \\
 &  & 0.02 & 3.53 & 56 & 2 & 3.50 & 56 & 2 & -0.9 & 0.0 & 0.0 \\
 &  & 0.05 & 4.86 & 88 & 14 & 4.87 & 88 & 14 & 0.2 & 0.0 & 0.0 \\
 &  & 0.1 & 9.23 & 204 & 53 & 9.08 & 204 & 53 & -1.7 & 0.0 & 0.0 \\
 &  & 0.3 & 33.36 & 899 & 171 & 32.81 & 896 & 170 & -1.7 & -0.3 & -0.6 \\
 &  & 0.5 & 35.86 & 940 & 178 & 35.38 & 947 & 179 & -1.4 & 0.7 & 0.6 \\
 &  & 0.8 & 2.52 & 74 & 16 & 2.50 & 74 & 16 & - & - & - \\
 &  & 1 & 2.36 & 71 & 14 & 2.35 & 71 & 14 & - & - & - \\
\bottomrule
\end{tabular}
}
\end{table}
As noted before, a key issue in the verification of object detection models
such as YOLOv3 is the fact that the architectures heavily rely on
downsampling inputs using MaxPool layers. The MaxPool layer is piece-wise
linear, bound propagation frameworks therefore need to employ convex
relaxations in order to model its behaviour. In our object detection
models, we replace the MaxPool with AvgPool layers and find that the
resulting networks yield comparable standard performance. However, since
the AvgPool function is linear and can be represented exactly in bound
propagation frameworks, we expect that the verification of those networks
is significantly more efficient than that of MaxPool-based models. To
verify this claim, we train a YOLOv3-tiny employing MaxPool instead of
AvgPool layers on the LARD dataset at a $64 \times 64$ resolution. The AvgPool model achieves an $\text{mAP}_{0.5:0.95}$ of $40.9\%$ and an
$\text{mAP}_{0.5}$ of $86.59\%$. Meanwhile, the MaxPool model achieves a
slightly higher $\text{mAP}_{0.5:0.95}$ of $41.67\%$ and a comparable
$\text{mAP}_{0.5}$ of $86.88\%$. Table~\ref{tab:ablation_study_pooling}
compares the verification of both models on a random subset of $50$
correctly classified images from the LARD test dataset. It is evident that
the AvgPool-based model is significantly easier to verify than the MaxPool
model. Verification runtimes are much lower for the AvgPool model and fewer
timeouts are encountered as a consequence. Verification times for the
MaxPool model are more than one order of magnitude higher except for very
small perturbation sizes.

\section{Complete Verification Performance}
\label{sec:complete_verification_performance}
In this section, we theorise about the performance gains achieved by our
optimal bounding method compared to the baseline when performing complete
verification as shown in Table~\ref{tab:comparison_iou_bounding_methods}.
While in incomplete verification tighter bounds often come at the cost of
longer verification times, in complete verification we would ideally expect
that tighter bounds allow the verifier to avoid exploring certain nodes in
the branch-and-bound (BaB) procedure, thereby improving overall efficiency.

However, it is important to recognise that tighter bounds do not always
guarantee faster complete verification. The computational
overhead introduced by tighter bounding can offset or even outweigh the
benefits of pruning nodes.

Within our object detection (OD) verification framework, when integrated
into a complete verification setting, we can fix the number of candidate
boxes we would consider, and if that number is exceeded our procedure
directly moves on to the splitting step. Under this assumption, once
bounds are propagated through the network, the OD verification-specific
computation takes constant time, independent of the number of anchor boxes.

To understand the time gains, we first define the following variables:  
\begin{itemize}
    \item $c_{\mathrm{opt}}$: time per call to our optimal OD bounding method,  
    \item $c_{\mathrm{base}}$: time per call to the baseline bounding method, with $c_{\mathrm{opt}} > c_{\mathrm{base}}$ (as tighter methods often incur higher per-call cost),  
    \item $c_{\mathrm{prop}}$: time for bound propagation before reaching the OD verification step,  
    \item $n$: total number of nodes explored in the BaB procedure by the baseline,  
    \item $s$: number of nodes avoided (pruned) by using our tighter bounds.  
\end{itemize}
We can approximate the total runtime of the baseline as:
\[
t_{\mathrm{base}} \approx n \cdot (c_{\mathrm{base}} + c_{\mathrm{prop}}),
\]
and the total runtime of the optimal method as:
\[
t_{\mathrm{opt}} \approx (n - s) \cdot (c_{\mathrm{opt}} + c_{\mathrm{prop}}).
\]
Defining the per-call overhead of the tighter bounds as $\Delta c := c_{\mathrm{opt}} - c_{\mathrm{base}}$, the total time saved can be expressed as:
\[
\Delta t = t_{\mathrm{base}} - t_{\mathrm{opt}} = -n \Delta c + s (c_{\mathrm{opt}} + c_{\mathrm{prop}}).
\]
This formulation reveals that if $c_{\mathrm{prop}} \gg \Delta c$, even a
small number of pruned nodes ($s$) can yield substantial runtime savings.
On the other hand, if $c_{\mathrm{prop}}$ is low, the overhead term $-n
\Delta c$ may dominate, potentially resulting in net slowdowns or only
marginal gains. There are also intermediate cases where the number of
pruned nodes, the overhead, and the propagation cost all balance each other
out, making patterns in the performance changes less clear.

\section{Proof of Theorem~\ref{th:iou_is_maximum}}
\label{sec:proof_iou_is_maximum}
Before presenting the proof for Theorem~\ref{th:iou_is_maximum}, we provide the analysis of
$\frac{\partial \text{IoU}}{\partial z_k}$ on the border of constraints,
since points where the gradient is 0 would have to be considered. We recall
Theorem~\ref{th:iou_is_maximum} which states that \textit{The IoU function
is maximum for a point obtained from within the set $C_s$}.
\begin{proof}
Let $M=(z^*_0,z^*_1,z^*_2,z^*_3)$ be an extreme of the function. It defines
on the $(z_0;z_2)$ plane the $M_x=(z^*_0,z^*_2)$ point, and on the
$(z_1;z_3)$ plane the $M_y=(z^*_1,z^*_3)$ point. Here we write $P_x := P^c_x \cup P^{\text{int}}_x \cup P^{gt}_x$ and
$P_y := P^c_y \cup P^{\text{int}}_y \cup P^{gt}_y$ for the candidate sets of
each plane. We observe that $M_x \in
P_x$ or that there is an $M'_x \in P_x$ that has equal IoU value. The case
for $M_y$ is analogous.\\

First, assume that $M_x \notin P_x$. Suppose $M_x$ lies within the interior
of the constrained region of the $(z_0;z_2)$ plane (which is defined by the
first and third constraints on the maximisation problem). We know that at
least one of the components is different from the ground truth, as
$(g_0$,$g_2) \in P_x$. Let us assume it's the first one. Then either
$z^*_0>g_0$, or $z^*_0<g_0$. If $z^*_0>g_0$, we have that $\frac{\partial
\text{IoU}}{\partial z_0}(M)<0$. And since we are within the constrained
region, the point $(z^*_0+\epsilon,z^*_1,z^*_2,z^*_3)$ still satisfies the
constraints, and has a bigger/smaller $IoU$ value (depending on the sign of
$\epsilon$) contradicting that $(z^*_0,z^*_2)$ is an extreme. If
$z^*_0<g_0$, we have that $\frac{\partial \text{IoU}}{\partial z_0}(M)>0$,
and we can apply the same reasoning.\\

Thus, we know that $M_x$ lies on the border of the region. Since it's also
not in $P_x$, then the following hold: $z^*_0\neq g_0$, $z^*_2\neq g_2$
(because the intersection of the ground truth coordinate and the border is
considered in $P_x$) and $M_x$ is not a corner of the region. Since $M_x$
lies on the border, we can define $\alpha^i(z_0,z_1,z_3) =
(z_0,z_1,\alpha(z_0),z_3)$, where $\alpha$ is the parameterisation of the
border in the $(z_0;z_2)$ plane.
In an analysis we present after the proof, we show how $\frac{\partial \text{IoU}
\circ \alpha^i}{\partial z_0} (z_0,z_1,z_3)$ behaves. In short, we prove
that in intervals contained within the border and the intersection with
either $z_0$ or $z_2$ with their respective ground truths, the derivative
is either always $0$ or never $0$. In the first case, any point within a
segment will have the same value, and thus in particular will have the same
IoU value as in either the intersection of $z_0$ or $z_2$ with their
respective ground truth, or the end of the segment due to the intersection
with another constraint (\ie, a corner). And all those points lie within $P_x$.
Thus, if $M_x$ lies within one of those segments, we could build $M'_x
\in P_x$ by changing the $(z_0,z_2)$ coordinates with the end of the
segment, and it will have the same IoU.
If $\frac{\partial \text{IoU} \circ \alpha^i}{\partial z_0}
(z^*_0,z^*_1,z^*_3) \neq 0$, we have that $\text{IoU} \circ \alpha^i(z^*_0
+ \epsilon,z^*_1,z^*_3)$ is either greater or smaller than $\text{IoU}
\circ \alpha^i(z^*_0 ,z^*_1,z^*_3)$ depending on the partial derivative
sign and the sign of $\epsilon$ we choose. And we know that $\alpha^i(z^*_0
+ \epsilon ,z^*_1,z^*_3)$ satisfies the constraints because: since it is a
parameterisation of one constraint, that one will always be satisfied, and
since it is not a corner of the region, the other constraint will also be
satisfied with a sufficiently small $\epsilon$. Thus, $M$ was not
an extreme of the function, resulting in a contradiction. Notice that
$\frac{\partial \text{IoU} \circ \alpha^i}{\partial z_0}
(z^*_0,z^*_1,z^*_3)$ is well defined, since $\frac{\partial \text{IoU}
\circ \alpha^i}{\partial z_0} (z^*_0,z^*_1,z^*_3)$ is non-defined only when
$z_0=g_0$ or $\alpha^i(z_0,z_1,z_3)_2=g_2$ and we already ruled out
those cases.

So then we proved that $M_x \in P_x$ or that there is an $M'_x \in P_x$
that has equal IoU value. An analogous reasoning can be applied to prove
that $M_y \in P_y$.
\end{proof}

\subsection{Partial Derivatives}

The following are the partial derivatives of IoU($b,g$) where
$b=(z_0,z_1,z_2,z_3)$ and $g$ is taken as a fixed number, which are needed
for the proofs that are presented next. These were presented in
\cite{Cohen+24}, though we fix here some of the typos they had.

\begin{equation}
\frac{\partial \text{IoU}(b)}{\partial z_i} = \frac{\partial \text{IoU}(b)}{\partial \mathbf{a}(b)} \frac{\partial  \mathbf{a}(b)}{\partial z_i} + \frac{\partial \text{IoU}(b)}{\partial \mathbf{a}(\mathbf{i}(b,g))} \frac{\partial \mathbf{a}(\mathbf{i}(b,g))}{\partial z_i}
\end{equation}

\begin{equation}
    \frac{\partial \text{IoU}(b)}{\partial \mathbf{a}(b)} = -\frac{\mathbf{a}(\mathbf{i}(b,g))}{\mathbf{a}(b\cup g)^ 2}
\end{equation}

Where $\mathbf{a}(b\cup g) := \mathbf{a}(b) + \mathbf{a}(g) - \mathbf{a}(\mathbf{i}(b, g))$

\begin{equation}
    \frac{\partial \text{IoU}(b)}{\partial \mathbf{a}(\mathbf{i}(b,g))} = \frac{\mathbf{a}(b\cup g) +  \mathbf{a}(\mathbf{i}(b,g))}{\mathbf{a}(b\cup g)^2}
\end{equation}

Let $c_i=(1-i)$ and $k_i=(2-i)$

\begin{equation}
    \frac{\partial  \mathbf{a}(b)}{\partial z_{i=0,2}} = - c_i (z_3-z_1)
\end{equation}

\begin{equation}
    \frac{\partial  \mathbf{a}(b)}{\partial z_{i=1,3}} = - k_i (z_2-z_0)
\end{equation}

\begin{equation}
\frac{\partial \mathbf{a}(\mathbf{i}(b,g))}{\partial z_{i=0,2}}= 
\begin{cases} 
0 & \text{if }  c_i z_i< c_i g_i  \\
- c_i \cdot (y_{max} - y_{min}) & \text{if }  c_i z_i > c_i g_i
\end{cases}
\end{equation}

\begin{equation}
\frac{\partial \mathbf{a}(\mathbf{i}(b,g))}{\partial z_{i=1,3}}= 
\begin{cases} 
0 & \text{if }  k_i z_i< k_i g_i  \\
- k_i \cdot (x_{max} - x_{min}) & \text{if }  k_i z_i > k_i g_i
\end{cases}
\end{equation}

\noindent \textbf{For $i\in \{0,2\}$:} 

If $ c_i z_i< c_i g_i$
\begin{equation}
\frac{\partial \text{IoU}(b)}{\partial z_{i=0,2}} = c_i  \frac{\mathbf{a}(\mathbf{i}(b,g))}{\mathbf{a}(b\cup g)^ 2}   (z_3-z_1) 
\end{equation}
Thus the derivative has the same sign as $c_i$

If $ c_i z_i > c_i g_i$

\begin{align}
&\frac{\partial \text{IoU}(b)}{\partial z_{i=0,2}} \nonumber \\
= &c_i \frac{\mathbf{a}(\mathbf{i}(b,g))}{\mathbf{a}(b\cup g)^2} (z_3 - z_1) - c_i \frac{\mathbf{a}(b\cup g) + \mathbf{a}(\mathbf{i}(b,g))}{\mathbf{a}(b\cup g)^2} (y_{max} - y_{min}) \nonumber \\
= &\frac{c_i}{\mathbf{a}(b\cup g)^2} [\mathbf{a}(\mathbf{i}(b,g)) (z_3-z_1) - (\mathbf{a}(b) + \mathbf{a}(g)) (y_{max} - y_{min})]\nonumber \\
= &\frac{c_i}{\mathbf{a}(b\cup g)^2} [(x_{max}-x_{min})(y_{max}-y_{min})  (z_3-z_1) - (\mathbf{a}(b) + \mathbf{a}(g)) (y_{max} - y_{min})]\nonumber \\
= &\frac{c_i \cdot (y_{max}-y_{min})}{\mathbf{a}(b\cup g)^2} [(x_{max}-x_{min})  (z_3-z_1)- \mathbf{a}(b) - \mathbf{a}(g)]\nonumber \\
= &\frac{c_i \cdot (y_{max}-y_{min})}{\mathbf{a}(b\cup g)^2} [(x_{max}-x_{min})  (z_3-z_1) \quad - (z_2-z_0) \cdot (z_3-z_1)- \mathbf{a}(g)]\nonumber \\
= &\frac{c_i \cdot (y_{max}-y_{min})}{\mathbf{a}(b\cup g)^2} [(x_{max}-z_2 + z_0 - x_{min})  (z_3-z_1) - \mathbf{a}(g)]\nonumber \\
\end{align}

Note that $x_{max} - z_2 <0$ and $z_0 - x_{min} <0$. Thus the derivative has the opposite sign of $c_i$

\noindent \textbf{Analogously, for $i\in \{1,3\}$:}

If $ k_i z_i< k_i g_i$
\begin{equation}
\frac{\partial \text{IoU}(b)}{\partial z_{i=1,3}} = k_i  \frac{\mathbf{a}(\mathbf{i}(b,g))}{\mathbf{a}(b\cup g)^ 2}   (z_2-z_0) 
\end{equation}
The derivative has the same sign as $k_i$

If $ k_i z_i > k_i g_i$

\begin{align}
&\frac{\partial \text{IoU}(b)}{\partial z_{i=1,3}} \nonumber \\
= &k_i \frac{\mathbf{a}(\mathbf{i}(b,g))}{\mathbf{a}(b\cup g)^2}   (z_2 - z_0) - k_i \frac{\mathbf{a}(b\cup g) + \mathbf{a}(\mathbf{i}(b,g))}{\mathbf{a}(b\cup g)^2} (x_{max} - x_{min}) \nonumber \\
= &\frac{k_i \cdot (x_{max}-x_{min})}{\mathbf{a}(b\cup g)^2} [(y_{max}-z_3 + z_1 - y_{min})  (z_2-z_0) - \mathbf{a}(g)]\nonumber \\
\end{align}

Note that $y_{max} - z_3 <0$ and $z_1 - y_{min} <0$. Thus the derivative
has the opposite sign of $k_i$

This allows us to build the following table, where $+$ means the function
increases in that interval, $-$ decreases and $nd$ means it is
non-differentiable in that point.

\begin{figure}[h]
    \centering
    \begin{tabular}{|c|c|c|c|c|c|}
        \hline
        $z$ & $-\infty$ & $g_0$ & & $g_2$ & $\infty$ \\
        \hline
        $\frac{\partial \text{IoU}}{\partial z_0}$ & $+$ & $nd$ & $-$ & $-$ & $-$ \\
        \hline
        $\frac{\partial \text{IoU}}{\partial z_2}$ & $+$ & $+$ & $+$ & $nd$ & $-$ \\
        \hline
    \end{tabular}
    \quad
    \begin{tabular}{|c|c|c|c|c|c|}
        \hline
        $z$ & $-\infty$ & $g_1$ & & $g_3$ & $\infty$ \\
        \hline
        $\frac{\partial \text{IoU}}{\partial z_1}$ & $+$ & $nd$ & $-$ & $-$ & $-$ \\
        \hline
        $\frac{\partial \text{IoU}}{\partial z_3}$ & $+$ & $+$ & $+$ & $nd$ & $-$ \\
        \hline
    \end{tabular}
    \caption{Variation of the IoU function given its partial derivatives.}
\end{figure}

\subsection{Analysing $\frac{\partial \text{IoU}}{\partial z_k}$ on the border of constraints}
\label{ssec:iou_on_the_border_of_constraints}

We will consider the general case:

\begin{equation*}
\label{newconstraints}
\begin{gathered}
\max_{z_0,z_1,z_2,z_3} \text{IoU}((z_0,z_1,z_2,z_3),g)  \\
\begin{aligned}
\text{s.t.}\quad 
L^x_1 &\leq z_0 + z_2 \leq U^x_1 \ \text{(a)} \\
L^y_1 &\leq z_1 + z_3 \leq U^y_1 \ \text{(b)} \\
L^x_2 &\leq z_2 - z_0 \leq U^x_2 \ \text{(c)} \\
L^y_2 &\leq z_3 - z_1 \leq U^y_2 \ \text{(d)}
\end{aligned}
\end{gathered}
\end{equation*}

We want to find the points where $\frac{\partial \text{IoU} \circ
\alpha^i}{\partial z_k} = 0$, for $k\in \{0,1\}$ and $i\in\{a,b,c,d\}$.

For $i\in \{a,b\}$, we can generalise the transformation for the
parameterisation as follows. \textbf{ $f^i(z_j)=K-z_j$}, taking $K\in
\{L^x_1, U^x_1, L^y_1, U^y_1 \} $. We also have \textbf{ $f^i(z_j)=K+z_j$}
when $i \in \{c,d\}$, and taking $K\in \{L^x_2, U^x_2, L^y_2, U^y_2 \}$.

So we will do the analysis for the cases $i \in \{a,c\}$ since the other
cases are analogous.

\begin{enumerate}
\item \textbf{\boldmath{$f^i(z_0)=K-z_0$:}} establishes the following
parameterisation: $\alpha(z_0,z_1,z_3)= (z_0,z_1,K-z_0,z_3)$.

Applying the chain rule, we have

\begin{align}
&\nabla \text{IoU} (\alpha(z_0,z_1,z_3)) \nonumber \\
= &\nabla \text{IoU}(z_0,z_1,z_2,z_3)_{|\alpha(z_0,z_1,z_3)} \cdot D\alpha(z_0,z_1,z_3)_{|(z_0,z_1,z_3)}
\end{align}

with

\[
D\alpha(z_0,z_1,z_3)_{|(z_0,z_1,z_3)} =  \begin{pmatrix}
1 & 0 & 0\\
0 & 1 & 0 \\
-1 & 0 & 0 \\
0 & 0 & 1 \\
\end{pmatrix}
\]

Thus, we have:

\begin{align}
&\frac{\partial \text{IoU} \circ \alpha}{\partial z_0} \nonumber \\
= &(\frac{\partial \text{IoU}(b,g)}{\partial z_0}_{|\alpha(z_0,z_1,z_3)} - \frac{ \partial \text{IoU}(b,g)}{\partial z_2}_{|\alpha(z_0,z_1,z_3)})
\end{align}

In order to ease the analysis, we define $A_x$ and $B_x$ as follows:

\[
A_x :=  \frac{\mathbf{a}(\mathbf{i}(b,g))}{\mathbf{a}(b\cup g)^ 2}   (z_3-z_1)>0 
\]

\begin{align}
B_x &:= \frac{\mathbf{a}(\mathbf{i}(b,g))}{\mathbf{a}(b\cup g)^2} (z_3 - z_1) \nonumber \\
	& \quad - \frac{\mathbf{a}(b\cup g) + \mathbf{a}(\mathbf{i}(b,g))}{\mathbf{a}(b\cup g)^2} (y_{max} - y_{min}) \nonumber \\
& < 0
\end{align}

To study the value of the derivatives, we will separate in the following cases:

\begin{enumerate}
    \item \textbf{Case $z_0<g_0$ and $z_2= K - z_0 < g_2$}: 
    
    \begin{align*}
    &\frac{\partial \text{IoU}(b,g)}{\partial z_0}_{|\alpha(z_0,z_1,z_3)} - \frac{ \partial \text{IoU}(b,g)}{\partial z_2}_{|\alpha(z_0,z_1,z_3)}\\ 
    & = A_x -  (-B_x) = A_x + B_x \\
    &= 2 \cdot \frac{\mathbf{a}(\mathbf{i}(b,g))}{\mathbf{a}(b\cup g)^2}   (z_3 - z_1)\nonumber \\
    & - \frac{\mathbf{a}(b\cup g) + \mathbf{a}(\mathbf{i}(b,g))}{\mathbf{a}(b\cup g)^2} (y_{max} - y_{min}) \nonumber\\
    &= \frac{(y_{max}-y_{min})}{\mathbf{a}(b\cup g)^2} [(2x_{max}- (K - z_0) \nonumber \\
    	& \quad + z_0 - 2x_{min})  (z_3-z_1) - \mathbf{a}(g)]
    \end{align*}

    Which equals 0 iff:

\begin{align*}
	0 &= (2x_{max}- K + z_0 + z_0 - 2x_{min})  (z_3-z_1) \\
		& \quad- \mathbf{a}(g)\\
	\mathbf{a}(g) &= (2x_{max}- K + 2 z_0 - 2x_{min})  (z_3-z_1) \\ 
	\mathbf{a}(g) &= [2 (x_{max}- x_{min}) - K + 2 z_0 ] \cdot (z_3-z_1)\\ 
	\mathbf{a}(g) &= [2 (K - z_0 - g_0) - K + 2 z_0 ] \cdot (z_3-z_1)\\ 
	\mathbf{a}(g) &= (2 K- 2z_0- 2 g_0- K+2z_0) \cdot  (z_3-z_1)\\
	\mathbf{a}(g) &= (K- 2 g_0) \cdot  (z_3-z_1)
\end{align*}
It follows that if $(K - 2g_0) \neq 0$, we have that $(z_3-z_1) = \frac{\mathbf{a}(g)}{2 (K - 2g_0)}$.

If $(K - 2g_0) = 0$, then it would require $\mathbf{a}(g)=0$ which is a case we do not consider.

\item \textbf{Case $z_0<g_0$ and $z_2= K - z_0 > g_2$}: 
    
\begin{align*}
    &\frac{\partial \text{IoU}(b,g)}{\partial z_0}_{|\alpha(z_0,z_1,z_3)} - \frac{ \partial \text{IoU}(b,g)}{\partial z_2}_{|\alpha(z_0,z_1,z_3)}\\ 
    & = A_x -  (-A_x) = A_x + A_x \\
    & = 2 A_x >0
\end{align*}

\item \textbf{Case $z_0>g_0$ and $z_2= K - z_0 > g_2$}

\begin{align*}
    &\frac{\partial \text{IoU}(b,g)}{\partial z_0}_{|\alpha(z_0,z_1,z_3)} - \frac{\partial \text{IoU}(b,g)}{\partial z_2}_{|\alpha(z_0,z_1,z_3)}\\
    & = B_x -  (-A_x) =A_x + B_x \\
\end{align*}

 Bear in mind that this case is not the same as the first one, as $x_{max}$ and $x_{min}$ differ. Thus, $A_x + B_x = 0$ iff:

\begin{align*}
	0 &= (2x_{max}- K + z_0 + z_0 - 2x_{min})  (z_3-z_1) - \mathbf{a}(g)\\
	\mathbf{a}(g) &= (2x_{max}- K + 2 z_0 - 2x_{min})  (z_3-z_1) \\ 
	\mathbf{a}(g) &= [2 (x_{max}- x_{min}) - K + 2 z_0 ] \cdot (z_3-z_1)\\ 
	\mathbf{a}(g) &= [2 (g_2 - z_0) - K + 2 z_0 ] \cdot (z_3-z_1)\\ 
	\mathbf{a}(g) &= (2 g_2 - 2z_0- K+2z_0) \cdot  (z_3-z_1)\\
	\mathbf{a}(g) &= (2 g_2- K) \cdot  (z_3-z_1)\\
\end{align*}
It follows that if $(2 g_2- K) \neq 0$, we have that $(z_3-z_1) = \frac{\mathbf{a}(g)}{ (2 g_2- K)}$.

The reasoning is similar to that of Case a). Notice that this condition does not depend on $z_0$ or $z_2$.

\item \textbf{Case $z_0>g_0$ and $z_2= K - z_0 < g_2$}

\begin{align*}
    &\frac{\partial \text{IoU}(b,g)}{\partial z_0}_{|\alpha(z_0,z_1,z_3)} - \frac{ \partial \text{IoU}(b,g)}{\partial z_2}_{|\alpha(z_0,z_1,z_3)}\\ 
    & = B_x -  (-B_x) = 2 B_x <0 \\
\end{align*}

\end{enumerate}

\item \textbf{\boldmath{$f^i(z_0)=K+z_0$:}} establishes the following parameterisation: $\alpha(z_0,z_1,z_3)= (z_0,z_1,K+ z_0,z_3)$. So now we have:

\[D\alpha(z_0,z_1,z_3)_{|(z_0,z_1,z_3)} =  \begin{pmatrix}
1 & 0 & 0\\
0 & 1 & 0 \\
1 & 0 & 0 \\
0 & 0 & 1 \\
\end{pmatrix}\]

Thus, we have:

\begin{align}
& \frac{\partial \text{IoU} \circ \alpha}{\partial z_0} \nonumber \\
= &(\frac{\partial \text{IoU}(b,g)}{\partial z_0}_{|\alpha(z_0,z_1,z_3)} + \frac{ \partial \text{IoU}(b,g)}{\partial z_2}_{|\alpha(z_0,z_1,z_3)} )
\end{align}

Again, we are going to split the analysis in the following cases:

\begin{enumerate}
    \item \textbf{Case $z_0<g_0$ and $z_2= K + z_0 < g_2$}: 
    
    \begin{align*}
    &\frac{\partial \text{IoU}(b,g)}{\partial z_0}_{|\alpha(z_0,z_1,z_3)} + \frac{ \partial \text{IoU}(b,g)}{\partial z_2}_{|\alpha(z_0,z_1,z_3)}\\ 
    & = A_x +  (-B_x) = A_x - B_x \\
    &= \frac{\mathbf{a}(b\cup g) + \mathbf{a}(\mathbf{i}(b,g))}{\mathbf{a}(b\cup g)^2} (x_{max} - x_{min}) > 0\\
    \end{align*}

There are no candidate points to add.

\item \textbf{Case $z_0<g_0$ and $z_2= K + z_0 > g_2$}:

\begin{align*}
    &\frac{\partial \text{IoU}(b,g)}{\partial z_0}_{|\alpha(z_0,z_1,z_3)} + \frac{ \partial \text{IoU}(b,g)}{\partial z_2}_{|\alpha(z_0,z_1,z_3)}\\ 
    & = A_x +  (-A_x) = 0 \\
\end{align*}

Similar to what happens in Cases a) and c) of the other family of borders,
having the gradient be identically $0$ means that we can move along the
border while keeping the IoU value constant until we reach one of the other
candidate points. Hence, this case introduces no new candidate points.

\item \textbf{Case $z_0>g_0$ and $z_2= K + z_0 > g_2$}

\begin{align*}
    &\frac{\partial \text{IoU}(b,g)}{\partial z_0}_{|\alpha(z_0,z_1,z_3)} + \frac{ \partial \text{IoU}(b,g)}{\partial z_2}_{|\alpha(z_0,z_1,z_3)}\\ 
    & = B_x +  (-A_x) =B_x - A_x \\
    &= - \frac{\mathbf{a}(b\cup g) + \mathbf{a}(\mathbf{i}(b,g))}{\mathbf{a}(b\cup g)^2} (x_{max} - x_{min}) < 0
\end{align*}

The gradient is never $0$.

\item \textbf{Case $z_0>g_0$ and $z_2= K + z_0 < g_2$}

\begin{align*}
    &\frac{\partial \text{IoU}(b,g)}{\partial z_0}_{|\alpha(z_0,z_1,z_3)} + \frac{ \partial \text{IoU}(b,g)}{\partial z_2}_{|\alpha(z_0,z_1,z_3)}\\ 
    & = B_x +  (-B_x) = 0 \\
\end{align*}

The reasoning is similar to that of case b).

\end{enumerate}
\end{enumerate}

With these results, we can see that for any particular case, either the
gradient is always zero, or never zero.

\section{Proof of Theorem~\ref{th:verification_algorithm_correctness}}
\label{sec:proof_verification_algorithm_correctness}
We first provide details of our proposed method for finding minimum and
maximum IoU values as Algorithm~\ref{alg:iou_bounds}. The algorithm
iterates through all $169$ critical points presented in
Section~\ref{ssec:optimal_iou_ibp_bounds}. Note that this is the number of
the IoU function's critical points for any given predicted box and a ground
truth box; this number is independent of the model architecture and other
parameters. When iterating through the points, our algorithm checks the
following:
\begin{itemize}
\item Does the point satisfy the constraints (as some of the candidate
points may not satisfy a particular combination of constraints)?

\item Is the point valid, \ie does it define a valid box satisfying
$z_0<z_2$ and $z_1<z_3$?
\end{itemize}
If these conditions are satisfied by a point, its IoU with the ground truth
is calculated and the maximum and minimum IoU values are updated.

\begin{algorithm}
\caption{Finding minimum and maximum IoU Values}
\label{alg:iou_bounds}
\begin{algorithmic}
\State \textbf{Input:} Set of critical points $C$, Set of constraints $S$
\State \textbf{Output:} Minimum and maximum IoU values
\State Initialise lower bound $\underline{\mathit{IoU}} \gets 1$
\State Initialise upper bound $\overline{\mathit{IoU}} \gets 0$
\For{each critical point $c \in C$}
    \If{$c$ satisfies constraints $S$ and is \textbf{valid}}
        \State Compute IoU value at $c$, denoted as $\mathit{IoU}(c)$
        \If{$\mathit{IoU}(c) > \overline{\mathit{IoU}}$}
            \State $\overline{\mathit{IoU}} \gets \mathit{IoU}(c)$
        \EndIf
        \If{$\mathit{IoU}(c) < \underline{\mathit{IoU}}$}
            \State $\underline{\mathit{IoU}} \gets \mathit{IoU}(c)$
        \EndIf
    \EndIf
\EndFor
\State \textbf{Return} $\underline{\mathit{IoU}}$, $\overline{\mathit{IoU}}$
\end{algorithmic}
\end{algorithm}

Algorithm~\ref{alg:verification_updated} represents a more detailed version
of Algorithm~\ref{alg:verification} which is helpful to understand the
proof of correctness for Algorithm~\ref{alg:verification} that we present
in the following.
\begin{algorithm}[b!]
\caption{Verification for Object Detection Models (Updated)}
\label{alg:verification_updated}
\begin{algorithmic}[1]
\State \textbf{Input:} Neural network $N$, input constraints $\zeta_x$, correctness constraints $\zeta_y$
\State \textbf{Output:} \textit{ROBUST} or \textit{NONROBUST} or \textit{UNKNOWN}

\Function{Verify}{$N, \zeta_x, \zeta_y$}
    \State bounds $\gets$ \Call{Bound\_Propagation}{$\zeta_x, N$}
    \State IoUBounds, scoreBounds, predClass $\gets$\\
    \hspace*{\algorithmicindent} \quad \Call{GetHighestBox}{bounds}
    \If{IoUBounds.min $\geq \tau_{\mathrm{iou}}$ \textbf{and}\\
    \hspace*{\algorithmicindent} \quad scoreBounds.lower $\geq \tau_{\mathrm{class}}$ \textbf{and}\\
    \hspace*{\algorithmicindent} \quad predClass $=$ \textit{class}}
        \State \Return \textit{ROBUST}
    \ElsIf{IoUBounds.max $< \tau_{\mathrm{iou}}$ \textbf{or}\\
    \hspace*{\algorithmicindent} \quad scoreBounds.upper $< \tau_{\mathrm{class}}$ \textbf{or}\\
    \hspace*{\algorithmicindent} \quad \Call{disagree}{predClass}}
        \State \Return \textit{NONROBUST}
    \Else
        \State \Return \textit{UNKNOWN}
    \EndIf
\EndFunction
\vspace{2mm}
\Function{GetHighestBox}{bounds}
    \State classLowerBound $\gets -\infty$
    \State classUpperBound $\gets -\infty$
    \For{each box $boxId$ in \textit{bounds}}
        \State \Call{ProcessLogits}{boxId}
        \State classLowerBound $\gets$ \textbf{\textit{max}}(classLowerBound,\\
        \hspace*{\algorithmicindent} \quad \quad boxId.classBounds.lower)
        \State classUpperBound $\gets$ \textbf{\textit{max}}(classUpperBound,\\
        \hspace*{\algorithmicindent} \quad \quad boxId.classBounds.upper)
    \EndFor
    \State minIoU $\gets 1$
    \State maxIoU $\gets 0$
    \For{each box $boxId$ in \textit{bounds}}
        \If{boxId.classBounds.upper $\geq$\\
        \hspace*{\algorithmicindent} \quad \quad classLowerBound}
            \State lowIoU, uppIoU $\gets$\\
            \hspace*{\algorithmicindent} \quad \quad \quad \Call{IoUBounds}{boxId.offsetBounds}
            \State minIoU $\gets$ \textbf{\textit{min}}(minIoU, lowIoU)
            \State maxIoU $\gets$ \textbf{\textit{max}}(maxIoU, uppIoU)
        \EndIf
    \EndFor
    \State \Return (minIoU, maxIoU), (classLowerBound,\\
    \hspace*{\algorithmicindent} \quad classUpperBound)
\EndFunction
\end{algorithmic}
\end{algorithm}

NOTE: We can use vectorised operations to first find the candidate boxes
before applying the expensive IoU bounding. To avoid computing bounds for
the IoU function in cases where verification is unlikely to succeed, we can
check whether the number of candidate boxes exceeds a threshold at this
step. If this is the case, we can continue branching in our complete
verification framework without having to compute the IoU bounds.

We recall Theorem~\ref{th:verification_algorithm_correctness} which states that \textit{\ioucert{} is correct}.
\begin{proof}[Proof (by induction)]
We prove by induction on the number of boxes $n$ in the \texttt{bounds}
input that the function \texttt{GetHighestBox} correctly computes:
\begin{itemize}
    \item The maximum class score lower bound across all boxes.
    \item The maximum class score upper bound across all boxes.
    \item The minimum and maximum IoU bounds over the set of candidate boxes.
\end{itemize}
\textbf{Base Case ($n = 1$):}  
If there is only one box:
\begin{itemize}
    \item The first \texttt{for} loop sets \texttt{classLowerBound} and \texttt{classUpperBound} to that single box’s \texttt{classBounds.lower} and \texttt{classBounds.upper}, respectively.
    \item The second \texttt{for} loop checks if the box's \texttt{classBounds.upper} is at least as large as the maximum lower bound (which it is, since there is only one box), and computes its IoU bounds.
\end{itemize}
Thus, the single box is correctly identified as the only candidate, and the
IoU and class score bounds are trivially correct.

\textbf{Inductive Hypothesis:}  
Assume that for any set of $k$ boxes ($1 \leq k \leq n$), \texttt{GetHighestBox} correctly:
\begin{itemize}
    \item Identifies the maximum class score lower and upper bounds.
    \item Selects all and only the candidate boxes whose \texttt{classBounds.upper} is at least the maximum lower bound.
    \item Computes the correct minimum and maximum IoU values over these candidates.
\end{itemize}

\textbf{Inductive Step ($n+1$ boxes):}  
Consider adding one more box $b_{n+1}$ to the existing $n$ boxes.

\begin{itemize}
    \item \textit{First \texttt{for} loop (updating class bounds):}  
    By the inductive hypothesis, after processing the first $n$ boxes, we have the correct \texttt{classLowerBound} and \texttt{classUpperBound}. When processing $b_{n+1}$, we update these bounds only if $b_{n+1}$ has a higher lower or upper class bound. Thus, after the loop, we have the correct maximum values over all $n + 1$ boxes.
    
    \item \textit{Second \texttt{for} loop (identifying candidates and IoU bounds):}  
    We check whether $b_{n+1}$ satisfies $b_{n+1}.\texttt{classBounds.upper} \geq \texttt{classLowerBound}$
    \begin{itemize}
    \item If \textbf{yes}, then in the worst case, all other candidates score at their lower bounds $\leq \texttt{classLowerBound}$, while $b_{n+1}$ scores at its upper bound $\geq \texttt{classLowerBound}$. Therefore, it is possible that $b_{n+1}$ is the highest scoring box, and it must be included as a candidate.
    
    \item If \textbf{no}, then the box attaining the maximum lower bound scores at least \texttt{class\allowbreak Lower\allowbreak Bound}, while $b_{n+1}$ can reach at most a score below \texttt{class\allowbreak Lower\allowbreak Bound}, making it impossible for $b_{n+1}$ to be the highest scoring box. Thus, it is correctly excluded.
\end{itemize}
    
    For all candidates, we compute their IoU bounds, updating \texttt{minIoU} and \texttt{maxIoU}. By inductive assumption, the existing candidates’ bounds are correct, and adding $b_{n+1}$ only refines these if it’s a valid candidate.
\end{itemize}

\textbf{Conclusion:}  
By induction, for any number of boxes $n$, \texttt{GetHighestBox}:
\begin{itemize}
    \item Correctly computes the maximum class score lower and upper bounds.
    \item Correctly identifies the candidate boxes.
    \item Correctly computes the minimum and maximum IoU values over these candidates.
\end{itemize}

Therefore, the overall \texttt{Verify} function’s branching logic
(\textit{ROBUST} / \textit{NONROBUST} / \textit{UNKNOWN}) is based on correctly
computed bounds, and \ioucert{} is correct.
\end{proof}

\section{Proof: Optimality of the LeakyReLU Relaxation}
\label{sec:proof_optimality_leaky_relu_relaxation}
LeakyReLU activations are piece-wise linear like ReLU activations and therefore
preserve desirable properties for NN verification such as the fact that splitting the input
interval to a LeakyReLU neuron at $x=0$ yields two subproblems where the
relaxation error for that neuron is zero. While Xu et al.~\cite{Xu+20b} extend the
Reluplex verifier to LeakyReLUs, the work more relevant to us is that of
Mellouki et al.~\cite{MelloukiIbnKhedherElYacoubi23} which extends the zonohedron-based
$\text{AI}^2$ verifier to LeakyReLUs. Let $\text{LeakyReLU}(x) = \max
\{\alpha x, x\}$ with $\alpha \in [0, 1]$ be a LeakyReLU activation with
concrete input bounds $x \in [l, u]$. If $u<0$ or $l>0$ the activation
function is said to be stable. In the case where $u<0$ it can be
represented exactly by $\text{LeakyReLU}(x)=\alpha x$ while in the case
where $l>0$ we have that $\text{LeakyReLU}(x)=x$. However, if $l < 0 < u$,
the behaviour of the LeakyReLU function is piece-wise linear. We therefore
need to employ linear lower and upper bounding functions to be able to
represent its behaviour in the linear bound propagation framework usually
employed by neural network verifiers. Linear lower and upper bounding
functions $f_\text{lower}(x), f_\text{upper}(x)$ (see Figure~\ref{fig:leaky_relu_relaxation}) are given by
\begin{align}
f_\text{lower}(x) &= \tilde{\alpha} x, \quad \tilde{\alpha} \in [\alpha, 1] \\
f_\text{upper}(x) &= \frac{u-\alpha l}{u-l}x + \frac{(\alpha-1)lu}{u-l}.
\end{align}
The bounds $\tilde{\alpha} \in [\alpha, 1]$ follow from the fact that
$\tilde{\alpha} x$ must be a valid lower bound for both parts of the
LeakyReLU function, namely $\alpha x$ and $x$.
Mellouki et al.~\cite{MelloukiIbnKhedherElYacoubi23} simply set $\tilde{\alpha}=\alpha$,
however, this ignores the fact that the tightness of the obtained bounds
can be significantly improved by either setting $\tilde{\alpha}$ to
minimise the overapproximation area or by optimising its slope as has been
proposed for ReLU activations \cite{Xu+21}. This has been demonstrated to
accelerate verification in various
works~\cite{Zhang+18,Singh+19a,HashemiKouvarosLomuscio21}. We derive the
overapproximation area and the corresponding optimal setting for
$\tilde{\alpha}$ depending on the concrete input bounds in the following.
\begin{figure}[htbp]
\centering
\begin{tikzpicture}
    \def\alphaValue{0.1}
    \def\l{-1}
    \def\u{2}
    \def\alphatilde{0.5}

    \pgfmathsetmacro{\fUpperSlope}{((\u - \alphaValue*\l)/(\u - \l))}
    \pgfmathsetmacro{\fUpperIntercept}{((\alphaValue - 1)*\l*\u)/(\u - \l)}

    \begin{axis}[
        title={LeakyReLU Relaxation ($\alpha=\alphaValue, l=\l, u=\u$)},
        xlabel={$x$},
        ylabel={$y$},
        axis lines=middle, 
        xmin=-2.5, xmax=2.5,
        ymin=-1.0, ymax=2.5,
        legend style={at={(-0.03,1)},anchor=north west,
            nodes={scale=0.7, transform shape},
            legend image post style={mark=*}
        },
        declare function={
            leaky(\x) = (\x > 0 ? \x : \alphaValue*\x); 
        },
        declare function={
            fUpper(\x) = \fUpperSlope*\x + \fUpperIntercept;
        },
        declare function={
            fLower(\x) = \alphatilde*\x;
        }
    ]
    
    
    \addplot [
        name path=fUpper, 
        blue, 
        dashed, 
        thick,
        domain=\l:\u, 
        samples=2 
    ] {fUpper(x)}
    node[pos=0.1, above, font=\small, anchor=south west, xshift=-3em, yshift=0em] {$f_\text{upper}(x)$};
    \addlegendentry{$f_\text{upper}(x)$}

    \addplot [
        name path=fLower, 
        red, 
        dashed, 
        thick,
        domain=\l:\u, 
        samples=2 
    ] {fLower(x)}
    node[pos=0.8, above, font=\small, anchor=north east, xshift=0em,yshift=-4.5em,] {$f_\text{lower}(x), \tilde{\alpha}=0.5$};
    \addlegendentry{$f_\text{lower}(x), \tilde{\alpha}=0.5$}

    
    \addplot [
        name path=leakynegative, 
        thick, 
        black,
        domain=\l:0, 
        samples=2
    ] {leaky(x)};
    \addlegendentry{$\text{LeakyReLU}(x)$}
    
    \addplot [
        name path=leakypositive,
        thick, 
        black,
        domain=0:\u, 
        samples=2,
        forget plot,  
    ] {leaky(x)}
    node[pos=0.7, above, font=\small, anchor=south east] {$\text{LeakyReLU}(x)$};

    
    \addplot [
        blue, 
        fill opacity=0.1
    ] fill between [
        of=fUpper and leakynegative, 
        soft clip={domain=\l:\u} 
    ];
    \addlegendentry{Upper relaxation area}

    \addplot [
        red, 
        fill opacity=0.1
    ] fill between [
        of=leakynegative and fLower, 
        soft clip={domain=\l:\u} 
    ];
    \addlegendentry{Lower relaxation area}
    \addplot [
        red, 
        fill opacity=0.1
    ] fill between [
        of=leakypositive and fLower, 
        soft clip={domain=\l:\u} 
    ];
    
    
    \draw [dashed, gray] (axis cs:\l, \pgfkeysvalueof{/pgfplots/ymin}) 
        -- (axis cs:\l, \pgfkeysvalueof{/pgfplots/ymax}) 
        node[above, pos=1, font=\small] {$x=l$};
        
    \draw [dashed, gray] (axis cs:\u, \pgfkeysvalueof{/pgfplots/ymin}) 
        -- (axis cs:\u, \pgfkeysvalueof{/pgfplots/ymax})
        node[above, pos=1, font=\small] {$x=u$};

    \end{axis}
\end{tikzpicture}
\caption{Visualisation of the lower and upper LeakyReLU bounding functions
for a sample $\tilde{\alpha}=0.5$. The lower overapproximation area is
shaded in red, the upper overapproximation area in blue.}
\label{fig:leaky_relu_relaxation}
\end{figure}

We recall Theorem~\ref{th:optimal_leaky_relu_relaxation} which states that
\textit{for $\text{LeakyReLU}(x) = \max \{\alpha x, x\}$ with $x \in [l, u]$, the local relaxation error is minimised by setting
\begin{equation}
f_\text{lower}(x) = \begin{dcases}
\alpha x \quad &\text{if} \: u < |l|,\\
x \quad &\text{else.}
\end{dcases}
\end{equation}}

\begin{proof}
To find the $\tilde{\alpha}$ which minimises the relaxation
error, we first compute the relaxation error induced by the lower bounding
function $f_\text{lower}$ both for $x<0$ and $x\geq0$.
\paragraph{Relaxation Error for $x < 0$}
For $x<0$ we compute the area enclosed by the LeakyReLU function and the
x-axis as
\begin{equation}
a_1 = \frac{1}{2} l (\alpha l)
\end{equation}
and the area enclosed by the lower bound for the LeakyReLU function and the
x-axis as
\begin{equation}
a_2 = \frac{1}{2} l (\tilde{\alpha} l).
\end{equation}
The total relaxation error is the difference between area enclosed by the
relaxation and that enclosed by the actual function, \ie
\begin{equation}
e_{x<0}(\tilde{\alpha}) = a_2 - a_1 = \frac{1}{2} l^2 (\tilde{\alpha} - \alpha).
\end{equation}

\paragraph{Relaxation Error for $x > 0$}
For $x>0$ we obtain that the area enclosed by the LeakyReLU function and
the x-axis is
\begin{equation}
a_3 = \frac{1}{2} u^2
\end{equation}
while that for the relaxation is
\begin{equation}
a_4 = \frac{1}{2} u (\tilde{\alpha} u).
\end{equation}
The total relaxation error can then be computed as the difference between the
area enclosed by the function and the area enclosed by the lower bound, \ie
\begin{equation}
e_{x>0}(\tilde{\alpha}) = a_3 - a_4 = \frac{1}{2} u^2 (1 - \tilde{\alpha}).
\end{equation}

\paragraph{Total Relaxation Error and Optimal Slope}
The total relaxation error $e$ can be calculated as:
\begin{align}
e(\tilde{\alpha}) &= e_{x<0}(\tilde{\alpha}) + e_{x>0}(\tilde{\alpha}) \\
&= \frac{1}{2} l^2 (\tilde{\alpha} - \alpha) + \frac{1}{2} u^2 (1 - \tilde{\alpha}) \\
&= \frac{1}{2} \left ( \left ( \tilde{\alpha} - \alpha \right ) l^2 + \left ( 1-\tilde{\alpha} \right ) u^2 \right ) \\
&= \frac{1}{2} \left ( \tilde{\alpha} l^2 - \alpha l^2 + u^2 - \tilde{\alpha} u^2 \right ) \\
&= \frac{1}{2} (u^2 - \alpha l^2) + \frac{1}{2} (\tilde{\alpha} l^2 - \tilde{\alpha} u^2) \\
&= \frac{1}{2} (u^2 - \alpha l^2) + \frac{1}{2} \tilde{\alpha} (l^2 - u^2) \label{eq:total_relaxation_error}
\end{align}
To determine the relation between the error $e$ and the relaxation slope
$\tilde{\alpha}$, we can compute the partial derivative
\begin{equation}
\frac{\partial e(\tilde{\alpha})}{\partial \tilde{\alpha}} = \frac{1}{2} (l^2 - u^2)
\end{equation}
from which it is obvious that if $u^2 < l^2$, the error $e$ grows with
growing $\tilde{\alpha}$ while for $u^2 > l^2$ a smaller $\tilde{\alpha}$
leads to larger errors. The error is minimised by selecting
the smallest valid $\tilde{\alpha}$ if $u^2 < l^2$ and the largest valid
$\tilde{\alpha}$ otherwise. The condition can be simplified by taking
its root on both sides and using the fact that $l < 0 < u$ (since we
otherwise do not employ relaxations). By exploiting the fact that $u>0$, we
obtain that $u^2 < l^2 \Leftrightarrow u < |l|$. For the lower relaxation
of the LeakyReLU function, it follows that to minimise the relaxation
error, we set
\begin{equation}
f_\text{lower}(x) = \begin{dcases}
\alpha x \quad &\text{if} \: u < |l|,\\
x \quad &\text{else}
\end{dcases}
\end{equation}
\end{proof}
This setting can either be frozen for the verification procedure or it can
be further optimised using gradient descent which we implement following
the procedure that Xu et al.~\cite{Xu+21} describe for ReLU activations.
\paragraph{Naive vs. Optimal Relaxation}
To illustrate the benefit of the optimal slope $\tilde{\alpha}$, we compute
the total overapproximation area for a concrete example.
\begin{example}
\label{ex:optimal_leaky_relu_relaxation}
The total
overapproximation area $e(\tilde{\alpha})$ is the area enclosed between the
upper bound $f_{\text{upper}}(x)$ and the lower bound $f_{\text{lower}}(x)
= \tilde{\alpha} x$. This area can be decomposed into two parts:
\begin{enumerate}
    \item $e_{\text{upper}}$: The area between the function $\text{LeakyReLU}(x)$ and the upper bound $f_{\text{upper}}(x)$.
    \item $e_{\text{lower}}(\tilde{\alpha})$: The area between the function $\text{LeakyReLU}(x)$ and the lower bound $f_{\text{lower}}(x)$.
\end{enumerate}
The total area is $e(\tilde{\alpha}) = e_{\text{upper}} +
e_{\text{lower}}(\tilde{\alpha})$. From Equation~\eqref{eq:total_relaxation_error}, $e_{\text{lower}}(\tilde{\alpha}) =
\frac{1}{2} l^2 (\tilde{\alpha} - \alpha) + \frac{1}{2} u^2 (1 -
\tilde{\alpha})$. The upper area $e_{\text{upper}}$ is the area of the
triangle formed by the vertices $(l, \alpha l)$, $(u, u)$, and $(0, 0)$,
which is given by $e_{\text{upper}} = -\frac{1}{2} lu (1-\alpha)$.

Consider a LeakyReLU neuron with $\alpha = 0.1$ and input bounds $x \in
[-2, 5]$. Here, $l=-2$ and $u=5$. Since $u=5 > |l|=2$, Theorem
\ref{th:optimal_leaky_relu_relaxation} states that the optimal relaxation
is achieved by setting $\tilde{\alpha}=1$ (the largest valid slope). The
``naive'' relaxation, as used by Mellouki et al.~\cite{MelloukiIbnKhedherElYacoubi23},
sets $\tilde{\alpha}=\alpha=0.1$.

We first calculate the constant upper approximation area
$e_{\text{upper}}$:
\begin{equation}
    e_{\text{upper}} = -\frac{1}{2} \cdot (-2) \cdot 5 \cdot (1-0.1) = 5 \cdot 0.9 = 4.5.
\end{equation}
Next, we calculate the lower approximation area $e_{\text{lower}}(\tilde{\alpha})$ as a function of $\tilde{\alpha}$:
\begin{align}
    e_{\text{lower}}(\tilde{\alpha}) &= \frac{1}{2} \cdot (-2)^2 \cdot (\tilde{\alpha} - 0.1) + \frac{1}{2} \cdot 5^2 \cdot (1 - \tilde{\alpha}) \\
    &= \frac{1}{2} \cdot 4 (\tilde{\alpha} - 0.1) + \frac{1}{2} \cdot 25 \cdot (1 - \tilde{\alpha}) \\
    &= 2 \cdot (\tilde{\alpha} - 0.1) + 12.5 \cdot (1 - \tilde{\alpha}) \\
    &= 2 \cdot \tilde{\alpha} - 0.2 + 12.5 - 12.5 \cdot \tilde{\alpha} = 12.3 - 10.5 \cdot \tilde{\alpha}.
\end{align}
The total overapproximation area is $e(\tilde{\alpha}) = e_{\text{upper}} + e_{\text{lower}}(\tilde{\alpha}) = 4.5 + 12.3 - 10.5\tilde{\alpha} = 16.8 - 10.5\tilde{\alpha}$.

Now we compare the two cases:
\begin{itemize}
    \item \textbf{Naive Relaxation ($\tilde{\alpha}=0.1$):} The total area is $e(0.1) = 16.8 - 10.5 \cdot 0.1 = 16.8 - 1.05 = \mathbf{15.75}$.
    \item \textbf{Optimal Relaxation ($\tilde{\alpha}=1$):} The total area is $e(1) = 16.8 - 10.5 \cdot 1 = 16.8 - 10.5 = \mathbf{6.3}$.
\end{itemize}
Choosing the optimal lower bound slope therefore reduces the total overapproximation area from $15.75$ to $6.3$, a reduction of $60\%$. This demonstrates the significant tightening achieved by selecting the slope $\tilde{\alpha}$ to minimise the relaxation area, which can lead to faster verification times as the bounds are propagated through a network.
\end{example}

\section{Scope, Applicability to Other Detectors, and Multi-Object Extension}
\label{sec:scope_and_extensions}
This section expands on the scope of \ioucert{} and discusses its applicability to
detector families beyond those evaluated in the main paper. It also presents a sketch of
how it could be extended from the single-object setting to full
multi-object detection, together with the associated obstacles.

\subsection{Applicability to Other Detector Families}
\label{ssec:applicability_other_detectors}
The coordinate transformation described in Section~\ref{ssec:coordinate_transformation} which is at the core of \ioucert{} and the optimal IoU bounds of
Section~\ref{ssec:optimal_iou_ibp_bounds} apply to any detector whose
box-decoding map $\psi \circ \phi$ is injective with a tractable inverse, that
is, whenever the decoding function is strictly monotonic in each predicted
offset (Appendix~\ref{sec:h_phi_functions}). This condition is satisfied by
the dense anchor-based heads of the SSD and YOLO families we evaluate, and
is independent of how the detector is trained. In particular, training-time
advances such as adaptive sample selection and alternative label-assignment
strategies (\eg ATSS~\cite{Zhang+20b}, PAA~\cite{KimLee20},
OTA~\cite{Ge+21}) modify the training target rather than the inference-time
decoding map. They are hence orthogonal to the verification problem we
address. \ioucert{} applies unchanged to dense anchor-based detectors
trained with such schemes.

For \textbf{two-stage detectors} such as Faster R-CNN, \ioucert{} applies
directly to the anchor-based localisation performed by the region proposal
network (RPN). Full end-to-end verification additionally requires reasoning
about proposal selection, RoIAlign/RoIPool, the second-stage heads, and
top-$k$/NMS filtering, which we leave for future work. For
\textbf{anchor-free detectors}, many heads still regress offsets or
distances from spatial locations, so the same principle applies whenever
the decoding map can be inverted or bounded. \textbf{Transformer-based
detectors} such as DETR introduce attention and set-prediction mechanisms
that are orthogonal to our IoU-bound contribution and remain challenging for
current verifiers in general.

\subsection{Extension to Multi-Object Detection}
\label{ssec:multi_object_extension}
In the single-object setting studied in the main paper, correctness depends
only on the bounds of the IoU between the object prediction and the ground
truth $\mathrm{IoU}(B_i, g)$. Extending \ioucert{} to multiple objects
additionally requires bounding the pairwise box-box overlaps used by NMS,
\ie deriving $\underline{J}_{ik} \le \mathrm{IoU}(B_i, B_k) \le
\overline{J}_{ik}$ for candidate pairs $i,k$. Given such bounds, the greedy
NMS loop can be certified by induction: if $\overline{J}_{ik} \le
\tau_{\mathrm{nms}}$ the boxes never suppress one another, whereas if
$\underline{J}_{ik} > \tau_{\mathrm{nms}}$ and the score ordering is
certified, the suppression is fixed. Once the score-threshold, score-order,
and pairwise-overlap decisions are all fixed, the NMS output is determined.

This naive extension is sound but expensive. It requires up to $O(n^2)$
pairwise certificates, each involving \emph{two} variable boxes rather than
one variable box and a fixed ground truth, which is substantially harder to
bound tightly. Moreover, whenever
$\underline{J}_{ik} \le \tau_{\mathrm{nms}} < \overline{J}_{ik}$ the NMS
decision is genuinely ambiguous, and the verifier must either branch or
return \textit{UNKNOWN}. Making this practical therefore requires
NMS-specific abstractions, clustering and dominance reasoning over candidate
boxes as well as dedicated branching strategies. We view NMS-aware verification
over the candidate boxes already tightly bounded by \ioucert{} as the most
promising first component to develop, and consider a full treatment
substantial enough for follow-up work.
\end{document}